\documentclass[11pt]{article}
\usepackage[left=1in,top=1in,right=1in,bottom=1in,letterpaper]{geometry}

\usepackage[utf8]{inputenc} %
\usepackage[T1]{fontenc}    %
\usepackage{hyperref}       %
\usepackage{url}            %
\usepackage{booktabs}       %
\usepackage{amsfonts}       %
\usepackage{nicefrac}       %
\usepackage{microtype}      %
\usepackage{xcolor}         %
\usepackage[most]{tcolorbox}
\usepackage{enumitem}
\usepackage{caption}
\usepackage{subfigure}
\usepackage{subcaption}
\usepackage[export]{adjustbox}
\usepackage{algorithm, algpseudocode}
\usepackage{tikz}
\usetikzlibrary{matrix,decorations.pathreplacing}
\usetikzlibrary{positioning, arrows.meta, shapes.geometric, fit}
\usepackage{float}
\usepackage[most]{tcolorbox}
\definecolor{linkColor}{HTML}{E74C3C}
\definecolor{pearcomp}{HTML}{B97E29}
\definecolor{citeColor}{HTML}{2980B9}
\definecolor{urlColor}{HTML}{1D2DEC}
\definecolor{conjColor}{HTML}{9ab569}
\usetikzlibrary{shadows}

\newcounter{exa}
\renewcommand{\theexa}{\arabic{exa}} %
\definecolor{namecolor}{RGB}{173,216,230}      %
\definecolor{relationcolor}{RGB}{204,232,204}  %
\definecolor{locationcolor}{RGB}{255,218,185}  %

\tcbset{
  myexample/.style={
    enhanced,
    colback=yellow!10!white,
    colframe=red!50!black,
    fonttitle=\scshape,
    titlerule=0pt,
    attach boxed title to top center={yshift=-2mm},
    boxed title style={colback=yellow!10!white, colframe=red!50!black},
    coltitle=red!50!black,
    drop shadow,
    title={\refstepcounter{exa}Example~\theexa:\quad #1},
    highlight math style={reset,colback=LightBlue!50!white,colframe=Navy}
  }
}

\newtcolorbox{texample}[1][]{myexample={#1}}

\usepackage[utf8]{inputenc} %
\usepackage{natbib}
\usepackage[T1]{fontenc}    %
\usepackage{titletoc}
\usepackage{hyperref}       %
\usepackage{url}            %
\usepackage{booktabs}       %
\usepackage{amsfonts,amsmath,amssymb,amsthm}       %
\usepackage{nicefrac}       %
\usepackage{microtype}      %
\usepackage{xcolor}         %
\usepackage{enumitem}
\usepackage{mathtools}

\usepackage{graphicx}
\usepackage{subcaption}
\usepackage[normalem]{ulem}
\usepackage{ragged2e}
\usepackage{cleveref}

\usepackage{bbm, dsfont}
\usepackage{wrapfig,bm,comment,color}
\usepackage{breakurl,epsfig,epsf,fmtcount,semtrans,multirow,multicol,boldline}
\usepackage{tcolorbox}
\usepackage{tikz}
\usepackage{pifont}

\usepackage{appendix}
\usepackage{xcolor}

\definecolor{darkred}{RGB}{150,0,0}
\definecolor{darkgreen}{RGB}{0,150,0}
\definecolor{darkblue}{RGB}{0,0,200}
\hypersetup{colorlinks=true, linkcolor=darkred, citecolor=darkgreen, urlcolor=darkblue}

\newtheorem{theorem}{Theorem}%

\newtheorem{assumption}{Assumption}

\newtheorem{lemma}{Lemma}
\newtheorem{corollary}{Corollary}
\newtheorem{proposition}{Proposition}
\newtheorem{definition}{Definition}

\def \endprf{\hfill {\vrule height6pt width6pt depth0pt}\medskip}

\newcommand{\cln}[1]{\red{}}

\newcommand{\bxi}{\boldsymbol{\xi}}

\newcommand{\gbl}{\g^{(\ell)}}

\newcommand{\beq}{\begin{equation}}
\newcommand{\ba}{\begin{align}}
\newcommand{\ea}{\end{align}}

\newcommand{\eeq}{\end{equation}}

\newcommand{\A}{{\mtx{A}}}

\newcommand{\B}{{{\mtx{B}}}}

\newcommand{\Ib}{{{\mtx{I}}}}

\newcommand{\Sb}{{{\mtx{S}}}}

\newcommand{\diag}[1]{\text{diag}(#1)}

\newcommand{\Lc}{{\cal{L}}}

\newcommand{\Pb}{{\mtx{P}}}

\newcommand{\Eb}{{\mtx{E}}}

\newcommand{\one}[1]{{\mathbbm{1}}(#1)}

\newcommand{\Db}{{\mtx{D}}}

\newcommand{\onebb}{{\mathbf{1}}}

\newcommand{\z}{{\vct{z}}}

\newcommand{\dist}[1]{\texttt{dist}(#1)}

\newcommand{\Ac}{\mathcal{A}}

\newcommand{\bte}{{\boldsymbol{\theta}}}

\newcommand{\Sc}{\mathcal{S}}

\newcommand{\blkx}[2]{\x_{#1, [#2]}}

\newcommand{\w}{\vct{w}}

\newcommand{\ab}{\vct{a}}
\newcommand{\bb}{\vct{b}}

\newcommand{\g}{{\vct{g}}}

\newcommand{\Tc}{\mathcal{T}}

\newcommand{\x}{\vct{x}}

\newcommand{\y}{\vct{y}}

\newcommand{\W}{\mtx{W}}

\definecolor{emmanuel}{RGB}{255,127,0}

\newcommand{\pb}{{\vct{p}}}

\newcommand{\R}{\mathbb{R}}

\newcommand{\Z}{\mathbb{Z}}

\newcommand{\E}{\operatorname{\mathbb{E}}}

\newcommand{\eb}{\vct{e}}

\newcommand{\vct}[1]{\bm{#1}}
\newcommand{\mtx}[1]{\bm{#1}}

\newcommand{\X}{{\mtx{X}}}

\newif\ifdraft
\drafttrue

\ifdraft
    
    \newcommand{\shaw}[1]{\textcolor{violet}{Shaw:#1}}
    
    \newcommand{\red}{\textcolor{darkred}}

\else     
    \newcommand{\shaw}[1]{}
    \newcommand{\red}{}

\fi

\newcommand{\query}{\mathsf{Q}}
\newcommand{\key}{\mathsf{K}}
\newcommand{\val}{\mathsf{V}}
\newcommand{\outp}{\mathsf{O}}

\newcommand{\mQuery}{\W_\query}
\newcommand{\mKey}{\W_\key}
\newcommand{\mValue}{\W_\val}
\newcommand{\mOutput}{\W_\outp}

\newcommand{\mKQ}{\W_{\key\query}}

\newcommand{\test}{\mathsf{test}}

\newcommand{\attn}{\mathsf{Attn}}

\newcommand{\hc}{\mathsf{GC}}

\newcommand{\poly}{\mathsf{poly}}

\newcommand{\unif}{\mathsf{Unif}}

\newcommand{\bzero}{\mathbf{0}}

\title{Transformers Provably Learn to Internalize Chain-of-Thought}
\author{
Yixiao Huang$^{1}\thanks{Correspondence: yixiaoh@berkeley.edu}$ \qquad
Hanlin Zhu$^1$ \qquad  
Zixuan Wang$^2$
\\
Jiantao Jiao$^1$
\quad
Stuart Russell$^1$
\quad 
Somayeh Sojoudi$^1$
\quad
Song Mei$^1$
\\
\\
$^1$UC Berkeley
\qquad
$^2$Princeton University
}
\date{}

\begin{document}

\maketitle
\begin{abstract}
Chain-of-Thought (CoT) prompting substantially improves the sample efficiency 
of transformers, reducing the complexity of tasks like parity learning from 
exponential to polynomial in the input length. However, generating explicit 
reasoning steps at inference is computationally expensive. Implicit 
Chain-of-Thought (ICoT) has emerged as a promising empirical remedy that 
trains models to internalize intermediate steps within their hidden states, 
but its theoretical foundations remain poorly understood. We give the first 
theoretical analysis of ICoT, proving that an 
$L$-layer transformer trained under our proposed Log-ICoT curriculum learns 
$k$-parity with $\poly(n)$ samples and $L = \log_2 k$ training stages. This 
matches the sample efficiency of explicit CoT while eliminating its 
inference overhead, and extends prior one-layer parity guarantees to 
multi-layer architectures. Compared to standard ICoT, which removes 
thinking tokens one at a time, Log-ICoT removes them in geometric chunks, 
reducing the number of stages from linear in $k$ to logarithmic. 
Experiments on multi-layer transformers confirm the theory and visualize 
how reasoning is progressively absorbed into deeper layers.
\end{abstract}

\section{Introduction}
\label{sec:intro}

Chain-of-Thought (CoT) reasoning \citep{wei2022chain} has become a 
cornerstone for enabling Large Language Models (LLMs) to solve challenging 
tasks. By generating explicit intermediate tokens, CoT decomposes complex 
problems into manageable sub-steps, significantly boosting performance. 
However, this reasoning power comes at a high cost: the explicit generation 
of thinking tokens substantially increases inference latency and 
computational overhead.

To address this, Implicit CoT (ICoT) \citep{deng2024explicit} has emerged 
as a promising paradigm. ICoT trains models to internalize reasoning steps 
within their hidden states by progressively removing intermediate steps 
during fine-tuning, eliminating the need for explicit token generation at 
inference time. 
Despite promising empirical evidence, the underlying mechanics of when and how models internalize these steps remain poorly understood. Furthermore, the standard ICoT curriculum, which removes intermediate steps 
one by one, scales linearly with the length of the reasoning chain, leading 
to inefficient training.

\paragraph{Parity learning as a testbed.}
We study these questions in the setting of $k$-parity learning, a classic 
problem that has emerged as a canonical testbed for the role of intermediate 
supervision in transformer training \citep{kim2024transformers, wen2024sparse,li2026dissecting}. 
The task is provably hard for finite-precision gradient-based methods 
without intermediate supervision: any algorithm using $\poly(n)$ samples 
and queries cannot achieve non-trivial accuracy 
\citep{shalev2017failures, wies2022sub, kim2024transformers}. With explicit 
CoT supervision, however, even one-layer transformers can learn to solve parity 
efficiently \citep{kim2024transformers, wen2024sparse}. This makes parity an ideal 
task for studying how CoT reshapes optimization to make expressible 
solutions reachable, and how ICoT preserves this advantage while 
eliminating the inference-time cost of explicit reasoning.
\begin{figure}[!t]
    \centering
    \subfigure{\includegraphics[width=0.9\textwidth]{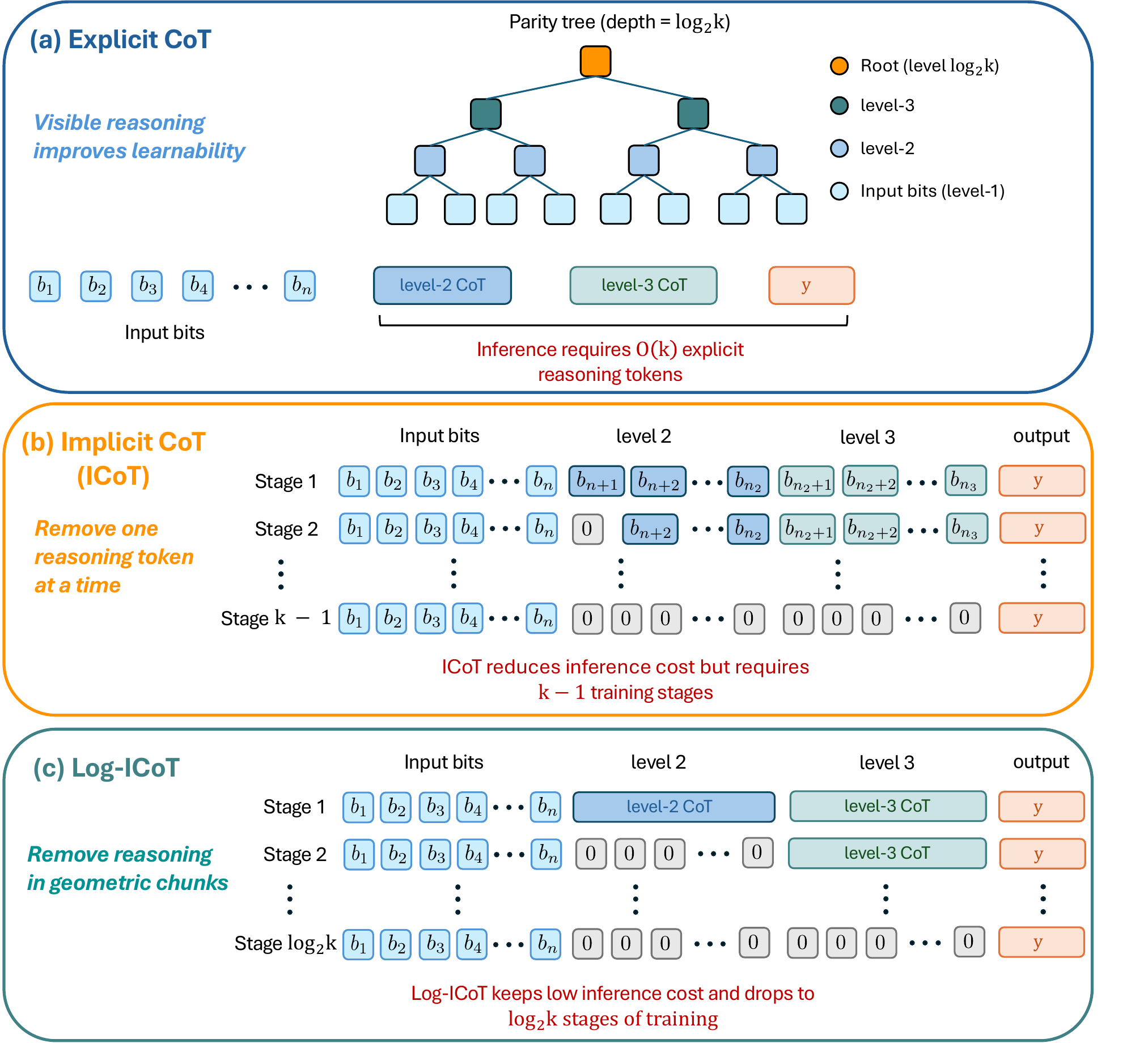}}
\caption{\small \textbf{Comparison of training paradigms on the $k$-parity task.} 
\textbf{(a)} Explicit CoT supervises every node of the parity tree, achieving 
sample-efficient learning at the cost of $\Omega(k)$ sequential reasoning 
tokens at inference. \textbf{(b)} Implicit CoT (ICoT)~\citep{deng2024explicit} 
internalizes reasoning into hidden states by removing intermediate tokens 
one at a time, eliminating the inference cost but requiring $k - 1$ training 
stages. \textbf{(c)} Our Log-ICoT curriculum removes tokens in geometric 
chunks aligned with the parity tree's levels, reducing the number of 
training stages to $\log_2 k$ while preserving inference efficiency. 
}
    \label{fig:teaser}
    \vspace{-3mm}
\end{figure}

\paragraph{Our contributions.}
We present the first theoretical analysis of ICoT, establishing that 
the sample efficiency of explicit CoT can be preserved when reasoning is 
internalized into hidden states. Specifically:
\begin{itemize}
    \item \textbf{Log-ICoT curriculum.} We introduce Log-ICoT, a new ICoT 
    curriculum that removes intermediate CoT steps in geometric increments 
    matching the parity tree's recursive structure (\Cref{fig:teaser}). 
    Compared to the linear-stage curriculum of standard 
    ICoT~\citep{deng2024explicit}, Log-ICoT requires only 
    $L = \log_2 k$ training stages.
    
    \item \textbf{Convergence guarantee for multi-layer ICoT.} 
    In \Cref{thm:icot-cvg}, we prove that an $L$-layer transformer trained 
    under Log-ICoT solves $k$-parity with $\poly(n)$ samples and $\log_2 k$ 
    gradient updates, matching the sample complexity of explicit 
    CoT~\citep{kim2024transformers, wen2024sparse} while producing the 
    answer in a single forward pass at inference (vs. $O(k)$ 
    sequential steps for autoregressive CoT generation). This extends 
    prior parity guarantees from one-layer to multi-layer architectures.
    \item \textbf{Resolving multi-layer training challenges.} 
Our analysis identifies two key challenges in analyzing the training 
dynamics of multi-layer transformers and resolves them via three design 
choices: \emph{gated connections} \citep{srivastava2015highway} to prevent 
representation collapse across layers, a \emph{customized causal attention mask} and 
\emph{integer quantization} of attention weights combined with a stage-wise 
learning rate to control error propagation across training stages.
    \item \textbf{Empirical verification.} 
    We verify our theoretical findings through numerical experiments on a 
    multi-layer transformer (\Cref{sec:exp}).
\end{itemize}

\subsection{Related Work}
\paragraph{Chain-of-thought and its variants.}

Chain-of-thought (CoT)~\citep{wei2022chain} can enhance LLMs' reasoning capabilities by letting models output the intermediate thought tokens. It can be encouraged only in the prompt~\citep{khot2022decomposed,zhou2022least} or be included in the training set~\citep{yue2023mammoth,yu2023metamath,wang2023math,shao2024deepseekmath}. A line of theoretical works also study the advantages of the CoT method via expressivity~\citep{liu2022transformers,feng2023towards,merrill2023expressive,li2024chain} or training dynamics~\citep{zhu2024towards,wen2024sparse,kim2024transformers}. A more recent line of work studies variants of CoT to improve the model performance or reduce the inference cost, including ICoT~\citep{deng2023implicit,deng2024explicit}, pause tokens~\citep{goyal2023think}, filler tokens~\citep{pfau2024let}, planning tokens~\citep{wang2023guiding}, token assorted~\citep{su2025token}, chain of continuous thought~\citep{hao2024training}, etc. Compared to CoT, its variants are much less explored, especially theoretically. \citet{london2025pause} theoretically show that the pause token can increase expressivity, and \citet{zhu2025emergence, zhu2025reasoning, gozeten2025continuous} theoretically show the advantage of continuous CoT. Our paper contributes to the field by providing the first theoretical results demonstrating the advantage of the ICoT method, i.e., it can bridge the gap between expressivity and learnability.

\paragraph{Training dynamics of transformers.} 
There is a very rich line of literature studying the optimization of transformer-based models~\citep{jelassi2022vision,bietti2023birth,mahankali2023one,fu2023can,zhang2024trained,li2024mechanics,huang2024context,ildiz2024self}. 
Many works focus on how certain attention patterns are formed during training~\citep{tian2023scan,tian2023joma,guo2024active}.
Another line of work focuses on various reasoning abilities or patterns of transformers through the lens of training dynamics~\citep{boix2023transformers,nichani2024transformers,wang2024transformers,renlearning,zhu2024towards,guo2025llms,huang2025transformers2,chen2024distributional,huang2025generalization, ma2026breaking}. 
The most related works to our setting are \citet{wen2024sparse,kim2024transformers}, which show that CoT enables much better sample efficiency when learning parity functions with secret indices. Our work follows the setting of \citet{kim2024transformers} and takes a further step to study the sample complexity of ICoT via training dynamics. Our work shows that ICoT enjoys both sample efficiency and low inference-time cost. Moreover, we provide an end-to-end analysis on multi-layer transformers, 
which is more practical and challenging than the one-layer transformer 
considered in \citet{wen2024sparse,kim2024transformers}. Our analysis also
builds on \citet{wang2025learning}, who studied attention-only multi-layer transformers. In contrast, the link function in our setting introduces 
additional cross-layer noise, which we suppress via gated 
connections.

\section{Preliminaries}
\label{sec:prelim}

\paragraph{Notations.} 
For any integer $N > 0$, $[N] := \{1, 2, \ldots, N\}$. We use lower-case 
and upper-case bold letters (e.g., $\ab$, $\A$) for vectors and matrices. 
Let $\eb_{d, i} \in \R^d$ denote the standard basis vector with a $1$ in 
the $i$-th coordinate and $\Ib_K \in \R^{K \times K}$ the identity. The 
multi-linear inner product of vectors $\x_1, \ldots, \x_r \in \R^d$ is 
defined as
\[
\langle \x_1, \ldots, \x_r \rangle := \sum_{i = 1}^d \prod_{j = 1}^r x_{j, i},
\]
which generalizes the standard inner product: 
$\langle \x_1 \rangle = \x_1^\top \onebb_d$ and 
$\langle \x_1, \x_2 \rangle = \x_1^\top \x_2$. Let \(q:\mathbb R\to\mathbb Z\) denote nearest-integer rounding and extend \(q\) entrywise to matrices.

\subsection{Task 
Settings} \label{sec:task-defn}

Assume $n \geq k \geq 2$ with $k = \Theta(n)$, and let $\Sc$ denote the 
set of all size-$k$ subsets of $[n]$. Following 
\citet{kim2024transformers, wen2024sparse}, we consider the $k$-parity 
problem: given a secret index set $S \sim \unif(\Sc)$ and an input 
$\bb \sim \unif(\{\pm 1\}^n)$, the label is the parity 
$y = p_S(\bb) := \prod_{j \in S} b_j$. Here $n$ is the input length and 
$k = |S|$ controls task hardness. We work in the finite-sample setting 
with a dataset of $B$ i.i.d.\ samples $\{(\bb^i, y^i)\}_{i = 1}^B$.

\paragraph{Task decomposition.}
We assume $k = 2^L$ for an integer $L$ and decompose the problem into a 
hierarchy of two-parities (\Cref{fig:teaser}). The task forms a 
complete binary tree of height $L$ with $2k - 1$ total nodes. Leaf nodes 
at level $1$ represent the input bits $\{b_j : j \in S \}$ in the secret set $S$, and the $k - 1$ internal 
nodes are indexed $b_{n+1}, \ldots, b_{n+k-1}$ from bottom to top, 
left to right. We write $n_\ell := n + k(1 - 2^{-(\ell - 1)})$ for the 
maximum index at level $\ell$ (so $n_1 = n$ and $n_{L+1} = T := n + k - 1$). For each position \(j\in[T]\), let \(p[j]\) be its parent in the parity tree if it exists, and set \(p[j]=\emptyset\) otherwise. For each internal node \(b_m\) (\(m>n\)), let \(c_1[m],c_2[m]\) denote its two children with \(c_1[m]<c_2[m]<m\). Set \(h[j]=1\) for \(j\le n\), and define \(h[m]\) for \(m>n\) by \(n_{h[m]-1}<m\le n_{h[m]}\). The recursive 
parity relation $b_m = b_{c_1[m]} \cdot b_{c_2[m]}$ holds at every 
internal node, and the root $b_T$ equals the target $y$. A concrete construction of the parity tree is given in \Cref{fig:tree-diagram} of \Cref{app:sec:parity}.

\paragraph{Loss and oracle.}
Let $f_\theta : \{\pm 1\}^n \to \R$ be a differentiable parametrized model. 
We consider the empirical squared loss
\[
\Lc_B(\theta) := \frac{1}{2B} \sum_{i = 1}^B \bigl(y^i - f_\theta(\bb^i)\bigr)^2,
\]
and evaluate performance by the population $L_2$ error
\[
\mathcal E(\theta) := \E_{S, \bb} \bigl[(p_S(\bb) - f_\theta(\bb))^2\bigr].
\]
An \emph{$\varepsilon$-approximate gradient oracle} for $\Lc_B$ is a 
(possibly randomized) map $\tilde \nabla \Lc_B$ satisfying 
$\|\tilde \nabla \Lc_B(\theta) - \nabla \Lc_B(\theta)\|_2 \leq \varepsilon$ 
for all queried $\theta$.

\begin{proposition}[\citet{kim2024transformers}, Theorem 2]\label{prop:sq}
Suppose $B = \Omega(n^{\nu})$ and $\|\nabla f_\theta\| = O(n^{\nu_1})$. 
Then there exists an $O(n^{-\nu_2})$-accurate gradient oracle $\tilde \nabla$ 
such that, with probability at least $1 - \exp(-\Omega(n))$ over the 
random sampling, the output $\theta(\Ac)$ of any iterative algorithm $\Ac$ 
making at most $O(n^{\nu_3})$ queries to $\tilde \nabla \Lc_B$ satisfies
$\mathcal E(\theta(\Ac)) \geq 1 - O(n^{-\nu_4})$, where 
$\nu = 4\nu_1 + 4\nu_2 + 2\nu_3 + 2\nu_4 + 1$.
\end{proposition}

\subsection{Model}\label{sec:model}

\paragraph{Architecture.}
We consider an $L$-layer simplified transformer following 
\citet{wang2025learning}. {To process the entire dataset in parallel, 
we stack the $j$-th bit across all $B$ samples into a vector 
$\bb_j = (b_j^1, \ldots, b_j^B)^\top \in \R^B$, so the input data 
$\Db = [\bb_1, \ldots, \bb_T] \in \R^{B \times T}$ is processed as a 
single batch using a shared positional encoding.}

\paragraph{Embedding layer.} 
We set $T := n + k - 1$ and embedding dimension $d := T + B(L + 1)$. 
For each $j \in [T]$, the embedding is
\[
\x_j = \Eb(\Db)_j = 
\begin{bmatrix}
\pb_j \\ \bb_j \\ \bzero_{B \cdot L}
\end{bmatrix} \in \R^d,
\]
where $\pb_j := \eb_{T, j} \in \R^T$ is a one-hot positional encoding and 
the final $B \cdot L$ dimensions are reserved for the residual stream 
\citep{elhage2021mathematical}.

\paragraph{Attention layer.}
We use a single-head attention per layer, reparameterizing the key-query 
matrix as $\mKQ := \mKey \mQuery^\top$ following 
\citet{tian2023scan, zhu2024towards, li2024mechanics}. We optimize only 
the upper $T \times T$ block $\W \in \R^{T \times T}$, fixing other 
entries to zero so that attention scores depend only on positional 
encodings. The value matrix is the identity:
\[
\mKQ = \begin{bmatrix}
\W & \bzero_{T \times B_L} \\
\bzero_{B_L \times T} & \bzero_{B_L \times B_L}
\end{bmatrix}, \quad \mValue = \Ib_d, \quad B_L := B(L + 1).
\]

\begin{definition}[Causal Self-Attention]
\label{sec:attn}
Given $\W \in \R^{T \times T}$, the causal self-attention module 
$\mathrm{Attn}(\cdot; \W) : \R^{d \times T} \to \R^{d \times T}$ is
\[
\mathrm{Attn}(\X; \W) = \X \cdot \mathbb{S}(\X^\top \W_{\mathsf{KQ}} \X + C),
\]
where $\mathbb{S}(\cdot)$ is the column-wise softmax and $C \in \R^{T \times T}$ 
is the {level-restricted} causal mask
\[
C_{j, m} = \begin{cases}
0 & {\text{if } m \leq n \text{ and } j < m,} \\
0 & {\text{if } m > n \text{ and } j \leq n_{h[m] - 1},} \\
-\infty & \text{otherwise.}
\end{cases}
\]
\end{definition}
\noindent A detailed discussion of the causal mask is provided in \Cref{sec:challenges}

\paragraph{MLP layer.}
The MLP layer applies a fixed link function $\phi : [-1, 1] \to [-1, 1]$ 
pointwise.

\begin{definition}[Link function]\label{def:link}
Following \citet{kim2024transformers}, $\phi$ is symmetric and satisfies 
$\phi(0) = -1$, $\phi(\pm 1) = 1$, so that 
$\phi\bigl(\frac{a + b}{2}\bigr) = a b$ for $a, b \in \{\pm 1\}$. We 
further assume $\phi$ is sufficiently smooth with 
$\phi'(0) = \phi'(\pm 1) = 0$, allowing the local Taylor expansion 
$\phi(t) = -1 + ct^2 + O(|t|^4)$ and $\phi'(t) = 2ct + O(|t|^3)$ around $0$ 
for some constant $c > 0$.
\end{definition}
\paragraph{Gated connections.}
We use gated connections \citep{srivastava2015highway} as a structured 
alternative to standard residual connections. Conventional residual updates 
can suffer a trade-off between representation collapse and gradient vanishing 
\citep{xie2023residual,zhu2024hyper}, and several variants modify the residual 
pathway to resolve this, including ResiDual~\citep{xie2023residual} and 
hyper-connections~\citep{zhu2024hyper}, both of which expand the residual 
stream into multiple parallel paths. We instead use a \emph{position-wise} 
gated connection in the spirit of GTrXL~\citep{parisotto2020stabilizing}, 
where each layer $\ell$ has a gate vector $\gbl \in \R^T$ controlling how 
the layer mixes the block update with the existing representation at each 
token position:
\[
\hc(\A, \B; \g) := \A \diag{\g} + \B (\Ib_T - \diag{\g}).
\]
\begin{definition}[Multi-layer Transformer]\label{def:transformer}
Let $L$ be the depth. The trainable parameters are 
$\bte := \{\W^{(\ell)}\}_{\ell \in [L]}$. 
Initializing $\X^{(0)} = \Eb(\Db) \in \R^{d \times T}$, the layer update is
\[
\X^{(\ell)} = \X^{(\ell - 1)} + \mOutput^{(\ell)} \hc\!\left(\phi(\attn(\X^{(\ell - 1)}; \W^{(\ell)})), \X^{(\ell - 1)}; \gbl\right), \quad \ell \in [L].
\]
The network output is $\Tc_\bte(\Db) = \X^{(L)}$. For $m \in [T]$, the 
output at position $m$ is
\[
\Tc_\bte(\Db)_m = \x_m + \sum_{\ell = 1}^L \mOutput^{(\ell)} \!\left( g_m^{(\ell)} \phi\Bigl({\sum_{j = 1}^{m - 1}} \sigma_j(\w_m^{(\ell)}) \x_j^{(\ell - 1)}\Bigr) + (1 - g_m^{(\ell)}) \x_m^{(\ell - 1)} \right),
\]
where $w_{j, m}^{(\ell)} = \pb_j^\top \W^{(\ell)} \pb_m$, and 
$\sigma_j(\w_m^{(\ell)})$ is the corresponding attention score.
The readout layer $\Psi_\ell \in \R^{d \times B}$ decodes the output as 
$f_\bte^{(\ell)}(\Db) = \Psi_\ell^\top \Tc_\bte(\Db)$.
\end{definition}

\paragraph{Block-shift output and readout.}
We fix $\mOutput^{(\ell)}$ as a block-shift operator that moves content 
from residual block $[\ell]$ to block $[\ell + 1]$:
\[
\mOutput^{(\ell)} = \begin{bmatrix}
\bzero_{T \times T} & \bzero_{T \times B_L} \\
\bzero_{B_L \times T} & \eb_{L + 1, \ell + 1} \eb_{L + 1, \ell}^\top \otimes \Ib_B
\end{bmatrix},
\]
and the readout $\Psi_\ell$ as a selector for block $[\ell + 1]$:
\[
\Psi_\ell = \begin{bmatrix}
\bzero_{T \times B} \\
\eb_{L + 1, \ell + 1} \otimes \Ib_B
\end{bmatrix}.
\]
{With this construction, $\Psi_{\ell'}^\top \mOutput^{(\ell)}$ is 
non-zero iff $\ell = \ell'$, so the readout at stage $\ell$ extracts 
exactly the content written by layer $\ell$. By 
\Cref{lemma:grad-upper-bounds}, gradients at well-trained layers are 
exponentially small, so we can focus on training one layer per stage.}

\section{Theoretical Results}
\label{sec:theory}
\subsection{Training Scheme}\label{sec:training-scheme}

\paragraph{Chain-of-Thought (CoT) for parity.}
The parity tree described in \Cref{sec:task-defn} gives a natural sequence 
of intermediate computations. We expand the input vector $\bb \in \{\pm 1\}^n$ 
to a length-$T$ sequence $\bb \in \{\pm 1\}^T$ by appending the internal 
node values $\bb_{n+1}, \ldots, \bb_T$, where each $\bb_m$ ($m > n$) is 
computed via the recursion $\bb_m = \bb_{c_1[m]} \cdot \bb_{c_2[m]}$. The 
CoT supervision trains the model to predict each $\bb_m$ at position $m$ \citep{kim2024transformers, wen2024sparse}.

\paragraph{Implicit Chain-of-Thought (ICoT).}
ICoT \citep{deng2024explicit} progressively removes intermediate CoT steps 
during training, internalizing the reasoning into hidden states so that at 
inference the model produces the answer directly without explicit reasoning. {In the original formulation, removed CoT steps shorten the sequence. We instead pad with zeros, keeping the length fixed at $T$. This does not 
sacrifice inference efficiency, since the model still produces the 
answer in a single forward pass. This design is theoretically grounded by \citet{merrill2026exact}, who show that padding tokens expand transformer expressivity by serving as extra parallel scratch positions for intermediate computation. Concretely, each parity-tree level is computed in 
parallel across positions with one level per transformer layer. This 
contrasts with autoregressive CoT generation, which would require 
$O(k)$ sequential steps.}
\paragraph{Log-ICoT curriculum.}
We adopt a Log-ICoT curriculum as in \Cref{fig:teaser} (c) that removes CoT steps in 
\emph{geometric increments}, mirroring the parity tree's level structure:
\begin{itemize}
    \item Stage $1$: Full CoT (all positions $\bb_{n+1}, \ldots, \bb_T$ shown).
    \item Stage $t$ ($2 \leq t \leq L$): Replace positions $\bb_{n+1}, \ldots, \bb_{n_t}$ 
    with padding (zeros), keeping $\bb_{n_t + 1}, \ldots, \bb_T$ shown.
    \item Stage $L$: All CoT steps except the final output are padded.
\end{itemize}
Since $L = \log_2 k$, the curriculum contains only a logarithmic number of 
stages, in contrast to the linear number used in standard ICoT 
\citep{deng2024explicit}.

\paragraph{Loss function.}
At stage $t$, we use the squared loss summed over the predicted positions:
\begin{align}\label{eq:loss-t}
\Lc^{(t)}(\bte) = \frac{1}{2B} \sum_{m = n_t + 1}^T 
\| \Psi_t^\top \Tc_\bte(\Db)_m - \bb_m \|^2.
\end{align}
This can be equivalently written as
\[
\Lc^{(t)}(\bte) = \frac{1}{2B} \sum_{i = 1}^B \| \bar y^i - f_\bte^\circ(\bb^i) \|^2,
\]
where $\bar y^i = (b_{n_t + 1}^i, \ldots, b_T^i)^\top$ collects the 
predicted positions for sample $i$ and 
$f_\bte^\circ(\bb^i)_m = (\Psi_t^\top \Tc_\bte(\Db)_m)_i$. This recovers 
the abstract form of \Cref{prop:sq} with a vector-valued predictor 
$f^\circ_\bte : \{\pm 1\}^n \to \R^{T - n_t}$.

\paragraph{Algorithm.}
At each stage \(t\), we draw an independent fresh batch of size \(B\) and
take one gradient step on \(\Lc^{(t)}\), followed by integer quantization
(Algorithm~\ref{alg:training}). Training proceeds sequentially through all
\(L=\log_2 k\) stages. Thus \(B\) denotes the per-stage batch size. Since
fresh samples are used at each stage, the total number of training samples
used by the curriculum is \(BL=O(B\log k)\).

\begin{algorithm}[t]
\caption{Training Algorithm (Log-ICoT Curriculum)}
\label{alg:training}
\begin{algorithmic}[1]
\Require Per-stage batch size \(B\), learning rate
$\eta_t = {K_n n_t^2} /({2c})$
\State Initialize \(\W^{(\ell)}(0)=0_{T\times T}\) for all \(\ell\in[L]\), $\bte: = \{\W^{(\ell)}\}_{ \ell \in [L]} $
\For{\(t=1,\ldots,L\)}
    \State Draw an independent fresh batch of \(B\) samples and construct the
    stage-\(t\) padded CoT data.
    \State $\bte(t) \leftarrow q \left (\bte(t - 1) - \eta \tilde \nabla_{\bte}\mathcal{L}^{(t)}(\bte{(t-1)}) \right) $ 
\EndFor
\State \Return \(\widehat\theta=\{\W^{(\ell)}(L)\}_{\ell=1}^L\)
\end{algorithmic}
\end{algorithm}

\paragraph{Evaluation.}
At test time, we use the Stage $L$ inference setup: random input bits 
filled in for positions $1, \ldots, n$, and zeros for the remaining 
$T - n = k - 1$ CoT positions. Concretely, given an independent test batch 
$\Db_{\test} = [\bb_1, \ldots, \bb_n, \bzero, \ldots, \bzero] \in \R^{B \times T}$ 
with $\bb_j \sim \unif(\{\pm 1\}^{B})$, the target is 
$(\y_{\test})_i = \prod_{j \in S} b_j^i$ for $i \in [B]$. Performance is 
measured by the population $L_2$ error $\mathcal E(\theta)$.

\subsection{Key Challenges} \label{sec:challenges}
There are two primary challenges in training multi-layer transformers. We discuss them in turn.
\paragraph{Representation Collapse. }
First, the input states collapse to a near-uniform value in late layers, 
causing the gradient to lose the ability to distinguish between token positions. For brevity, we drop the layer superscript in the following lemma. The result applies to 
any attention layer with weight matrix $\W$.
\begin{lemma} [Representation collapse of input states] \label{lemma:constant-input} 
For $B =\poly(n)$, with probability $1 - \exp(-n^{\epsilon / 16})$  over random sampling of the input data $\bb_1, \dots, \bb_n$, it holds that, for all $n^{\epsilon / 8} \leq m \leq n$:
\[
    \| \phi(\frac{1}{m -1}\sum_{\alpha \in [m - 1]} \bb_\alpha ) + \onebb_B \|_\infty \leq O(n^{-\epsilon/16}).
\]
In particular, when $\W$ is $\bzero_{T \times T}$, for all $n^{\epsilon / 8} \leq m \leq n$, we have
\[
   \|  \phi(\sum_{j = 1}^{m - 1} \sigma_j(\w_m) \bb_j) + \onebb_B \|_\infty =  \| \phi(\frac{1}{m - 1}\sum_{j \in [m - 1]} \bb_j ) + \onebb_B \|_\infty \leq O(n^{-\epsilon/16})
\]
\end{lemma}
The proof is given in \Cref{app:sec:proof-constant-input}, which simply adopts a concentration bound argument as the random input data has mean $0$ and is bounded. By \Cref{lemma:constant-input}, with uniform attention the input to the next 
layer collapses to approximately $-\onebb_B$. In a vanilla residual transformer, 
this collapse propagates: every CoT position's residual stream contains the 
noisy near-constant $-1$, drowning out the $O(n^{-2})$ token-distinguishing 
signal in subsequent layers' gradients with $\Omega(n^{-1})$ noise from 
contractions that should cancel by parity but instead accumulate.

To resolve this, we incorporate gated connections \citep{srivastava2015highway, parisotto2020stabilizing} to replace the residual 
connections. In the original 
formulation, the gating weights are input-dependent. In our setting, we instead \emph{prescribe} them 
from the parity tree's recursive structure. Specifically, for any 
$\ell \in [L]$, we set the gating weights as: 
\begin{align} \label{eq:gating-weights}
g_m^{(\ell)} =
\begin{cases}
1, & n_\ell + 1 \le m \le n_{\ell+1},  \\
0, & \text{otherwise}.
\end{cases}
\end{align}
Intuitively, layer $\ell$'s gating activates only at level-$(\ell + 1)$ 
positions $m \in [n_\ell + 1, n_{\ell + 1}]$, where the layer is responsible 
for producing the corresponding CoT predictions. At all other positions, 
the gating mutes the block update, leaving the residual stream to 
propagate prior content via the block-shift in $\mOutput^{(\ell)}$. 
As a consequence, training at each stage focuses gradient signal on a 
single level of the parity tree, avoiding representation collapse from 
uniform attention (\Cref{lemma:constant-input}) at lower-level positions 
where the loss should not yet be computed. Fixing the gates rather than 
learning them is what makes the per-stage training dynamics tractable. Each layer's gradient signal is structurally isolated to its assigned 
parity-tree level, enabling the layer-wise convergence analysis in 
\Cref{thm:icot-cvg}.

\begin{figure}[t]
    \centering

    \begin{minipage}[t]{0.34\textwidth}
        \centering
        \vspace{0pt}
        \includegraphics[height=3.5cm]{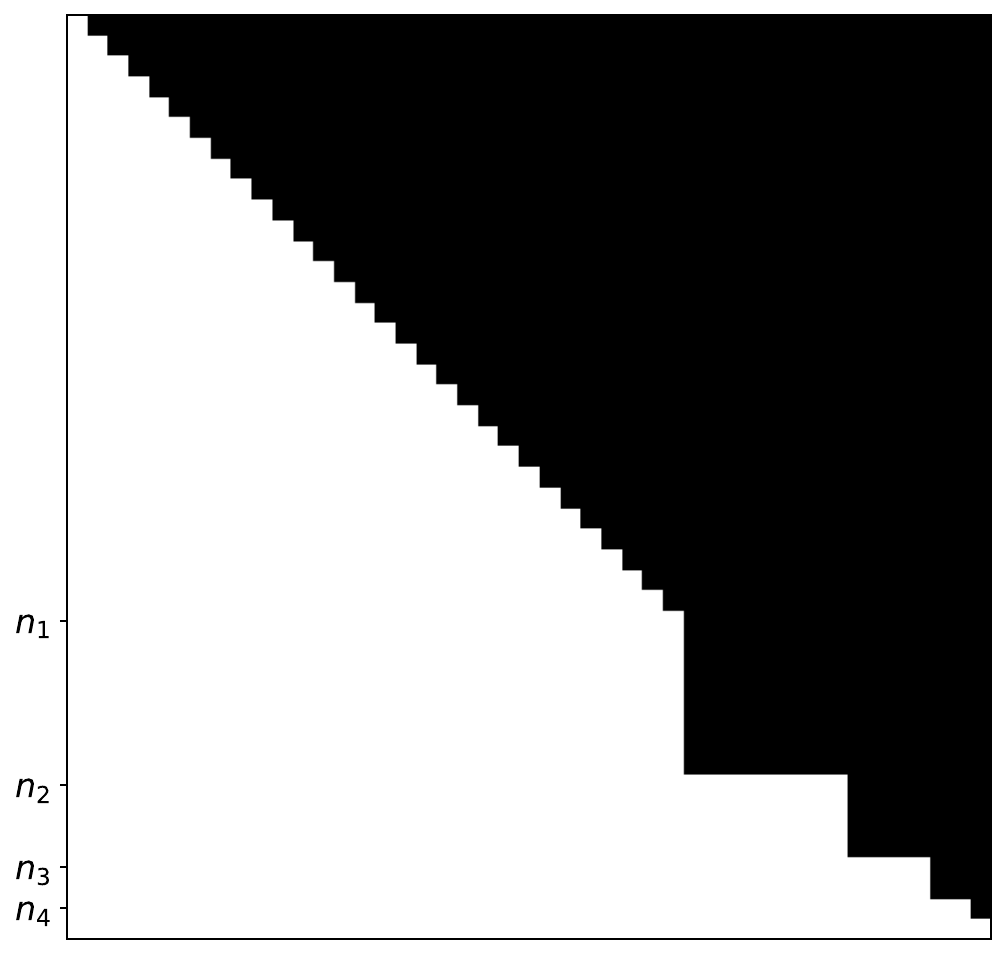}

        \vspace{2mm}
        {\small {(a) Customized attention mask.}}
    \end{minipage}
    \hfill
    \begin{minipage}[t]{0.62\textwidth}
        \centering
        \vspace{0pt}
        \includegraphics[height=3.5cm]{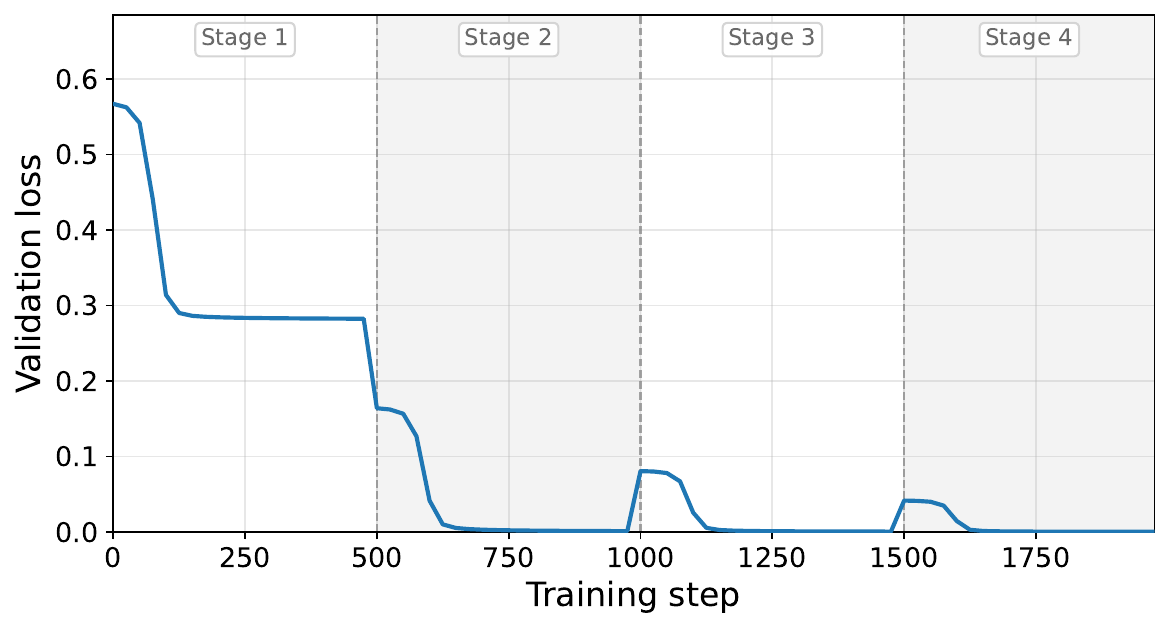}

        \vspace{2mm}
        {\small {(b) Validation loss under the Log-ICoT curriculum.}}
    \end{minipage}

    \vspace{2mm}
    \caption{\small
    \textbf{Attention mask and training dynamics.}
    Left: illustration of the customized attention mask, where each intermediate state
    $\x_m$ at CoT positions $(m > n)$ only depends on tokens $\x_j$ up to the
    previous level, i.e., $j \leq n_{h[m] - 1}$.
    Right: validation loss of a $4$-layer transformer trained under the Log-ICoT
    curriculum on the $k$-parity task $(n=30, k=16)$; dashed vertical lines mark
    the four training stages.
    }
    \label{fig:mask-loss}
\end{figure}

\paragraph{Error propagation. }As we adopt a multi-stage training curriculum, small approximation errors in earlier-stage predictions can compound through the residual stream and overwhelm the gradient signals. To address this, we introduce two modifications to the vanilla transformer model, following \citet{kim2024transformers}. 

{First, we revise the causal attention mask as in \Cref{sec:attn} so that each intermediate 
state $\x_m$ depends only on tokens at strictly lower levels of the parity 
tree. This isolates the 
gradient signal at the correct child positions for our analysis
(\Cref{lemma:base,lem:inductive-step}). A visualization of the revised mask can be found in \Cref{fig:mask-loss}(a).}
Second, we quantize the attention weights after every gradient update by rounding each 
entry of $\{\W^{(\ell)}\}_{\ell \in [L]}$ to the nearest integer:
\[
\W^{(\ell)}(t + 1) = q\!\left( \W^{(\ell)}(t) - \eta \nabla_{\W^{(\ell)}} \Lc^{(t)}(\bte(t)) \right),
\]
where $q : \R \to \Z$ is the nearest-integer operator. Integer-based 
quantization methods are widely used in practice to improve training and 
inference efficiency \citep{jacob2018quantization, wu2020integer}, and have 
also been used in theory to control error propagation in autoregressive generation
\citep{kim2024transformers}. {In our multi-stage setting, quantization 
plays an additional role: it locks the weights of previously trained layers. 
Once a layer is trained, the gradients with respect to its weights at 
subsequent stages are exponentially small (\Cref{lemma:grad-upper-bounds}), 
so quantization rounds each update back to the trained value. This allows 
us to analyze stage $t$ in isolation by treating layers $1, \ldots, t-1$ as 
fixed.}

\subsection{Main Theorem}\label{sec:main-theorem}

{We work in the regime $k = \Theta(n)$, so $L = \log_2 k = \Theta(\log n)$ 
and $T = n + k - 1 = \Theta(n)$. The bounded-gradient guarantee 
$\|\nabla_\bte f_\bte^{(t)}(\bb)\| = O(n^g)$ (\Cref{lemma:bounded-gd}) 
verifies the assumptions of \Cref{prop:sq}. Without intermediate 
supervision, no iterative algorithm using $\poly(n)$ samples and queries 
can achieve non-trivial accuracy on $k$-parity. The following theorem 
shows that Log-ICoT circumvents this barrier.}

\begin{theorem}[Log-ICoT]\label{thm:icot-cvg}
Consider an $L = \log_2 k$-layer transformer. Suppose the per-stage batch size satisfies $B = \Omega(n^{2 + \epsilon})$ for some constant $\epsilon > 0$, the 
gradient oracle is $O(n^{-2 - \epsilon/8})$-accurate, and at each stage 
$t \in [L]$, the learning rate is set to $\eta_t = \frac{K_n n_t^2}{2c}$, where 
$K_n := \lceil n^{\epsilon/16} \rceil$ and \(c>0\) is the link function constant in \Cref{def:link}. Then for sufficiently large $n$, with 
probability $1 - \exp(-\Omega(n^{\epsilon/2}))$ over the random sampling of 
the training data, the output $\hat \bte = \bte(L)$ of \Cref{alg:training} 
satisfies
\[
\| f_{\hat \bte}^{L} (\Db_{\test})_T - \y_{\test} \|_\infty 
\;\leq\; \exp(-\Omega(n^{\epsilon/16})).
\]
\end{theorem}

\subsection{Proof Sketch}
\label{sec:proof-sketch}

The proof proceeds by induction on the training stage $t = 1, \ldots, L$, 
showing that after stage $t$, layers $1, \ldots, t$ are \emph{well-trained}, 
i.e., their attention weights concentrate on the correct two children of 
each parity node, and the corresponding hidden states approximate the 
ground-truth parities up to additive error $\exp(-Cn^{\epsilon/16})$. We 
outline the five ingredients below. \Cref{fig:tree-diagram} (Right)
in \Cref{app:imp-details} illustrates the residual-stream evolution on a 
worked example with $n = 8$, $k = 4$.

\paragraph{Concentration of cross-correlations.}
We begin with a probabilistic ingredient that underpins the entire analysis. 
The choice $B = \Omega(n^{2+\epsilon})$ is dictated by a Hoeffding-type 
bound (\Cref{lemma:trivial-noise}, following \citet{kim2024transformers}): 
with high probability, every non-trivial empirical cross-correlation 
$\langle \bm{b}_{j_1}, \ldots, \bm{b}_{j_r} \rangle / B$ between parity 
terms of order $r \leq 4$ is at most $\kappa = O(n^{-1-\epsilon/4})$. 
This controls the noise in every gradient computation that follows.

\paragraph{Stage 1 (\Cref{lemma:base}).}
With concentration in hand, we turn to the base case. All weights are 
initialized to zero, so layer 1's attention is uniform. 
A direct expansion of 
$\partial \mathcal{L}^{(1)} / \partial w^{(1)}_{j, m}$ for 
$m \in [n + 1, n_2]$ using the link function's Taylor series isolates a 
signal term of order $n^{-2}$ at the two children indices 
$c_1[m], c_2[m]$, with all other 
contributions bounded by $O(n^{-2-\epsilon/8})$ via the concentration 
above. With learning rate $\eta_1 = \Theta(K_n n^2)$, 
one gradient step followed by integer quantization yields 
$w^{(1)}_{c_1[m], m} = w^{(1)}_{c_2[m], m} = K_n$ and 
zero elsewhere, producing softmax weights concentrated near $1/2$ on the 
two correct children. The link function then maps the resulting average to 
the correct parity at every level-2 position, with error $\exp(-Cn^{\epsilon/16})$.
\paragraph{Stage $t+1$ (\Cref{lem:inductive-step}).}
The inductive step extends this argument upward. Assume the well-trained state holds at stages $1, \ldots, t$. The same gradient expansion goes through at layer $t+1$, but now operating on the \emph{hidden states} $\bm{x}^{(t)}_{j, [t+1]}$ produced by the previous layers rather than raw inputs. \Cref{lemma:hidden-state} characterizes these hidden states: by the gated connection design (Eq.~\eqref{eq:gating-weights}), $\bm{x}^{(t)}_{j,[t+1]}$ equals $\bm{b}_j$ uniformly up to exponentially small error, so the concentration argument carries over with controlled propagated error.
\paragraph{Layer freezing across stages.}
For the induction to compose, we further need the earlier stages to remain 
undisturbed. A key feature of Log-ICoT is that once a layer has been trained, 
it stays trained: at well-trained positions, 
$\phi'(\hat{\bm{z}}) = O(\exp(-Cn^{\epsilon/16}))$ because $\hat{\bm{z}}$ 
is exponentially close to a point where $\phi'$ vanishes 
(Definition~\ref{def:link}). This makes the gradient at every previously 
trained weight exponentially small, so the quantization operator $q(\cdot)$ 
rounds each update back to zero, leaving the weight unchanged. This argument is  
formalized in \Cref{lemma:grad-upper-bounds}. Consequently, we can 
analyze each layer in isolation, treating earlier layers as fixed.

\paragraph{Bounding the prediction error.}
With the induction in place, it remains to control the propagated error at 
inference. Let $\epsilon_t = \max_m \| \bm{x}^{(t)}_{m, [t+1]} - \bm{b}_m \|_\infty$ 
denote the worst-case approximation error at stage $t$. The forward pass 
induces the recursion
\[
\epsilon_{t+1} \leq C_2 \bigl( 2 \exp(-Cn^{\epsilon/16}) + \epsilon_t \bigr)^2,
\]
which preserves $\epsilon_t \leq \exp(-Cn^{\epsilon/16})$ across all 
$L = \log_2 k$ stages. At the final stage $t = L$, the readout at position 
$m = T$ recovers $\y_{\text{test}}$ up to error $\exp(-Cn^{\epsilon/16})$, giving the claimed test error.

\section{Experiments}\label{sec:exp}
\begin{figure}[!t]
    \centering
    \subfigure
    {\includegraphics[width=0.7\textwidth]{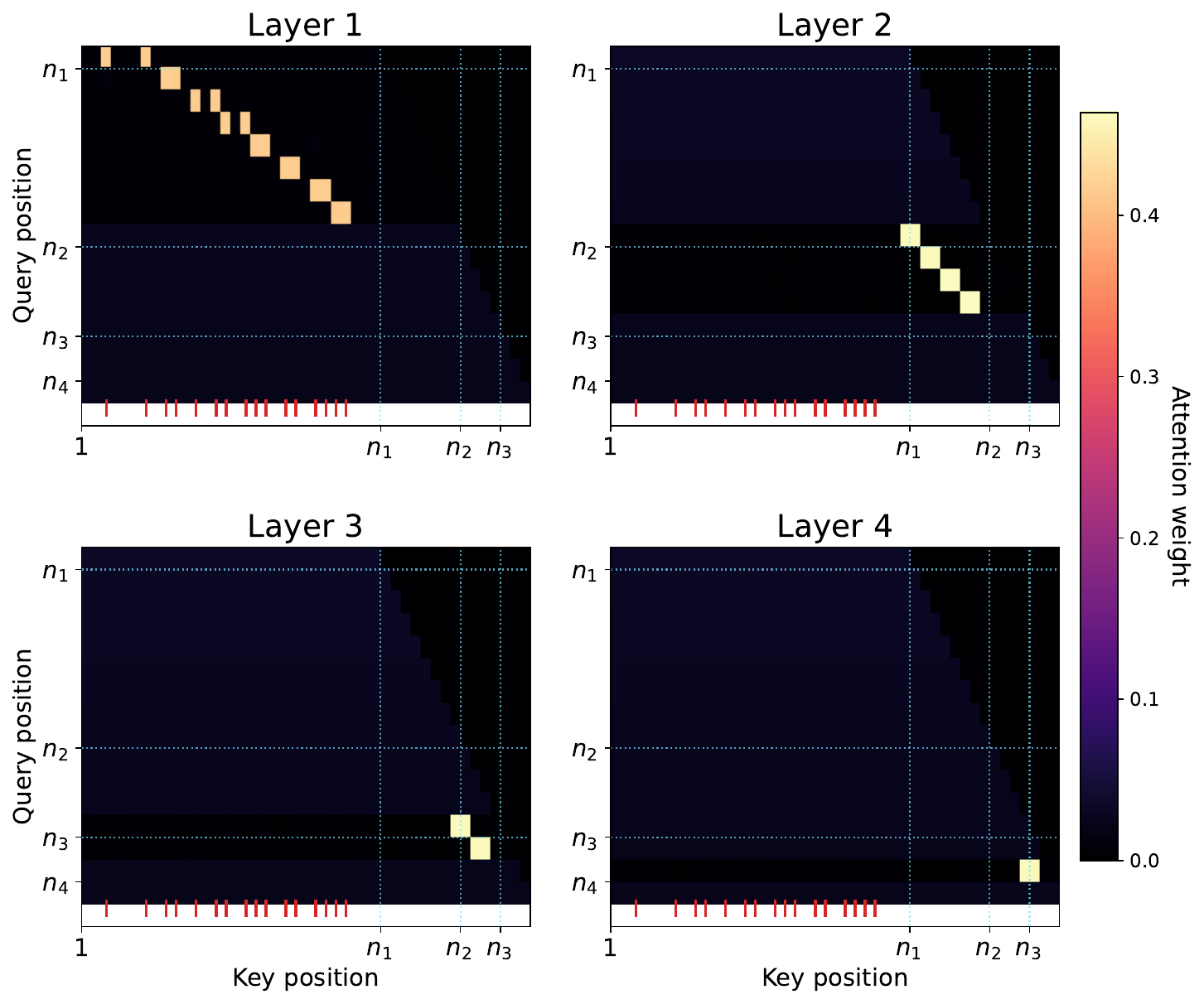}}
    \vspace{-2mm}
\caption{\small \textbf{Layer-wise attention maps of the trained 4-layer 
transformer at the final stage.} Each panel shows the softmax attention 
weights at one layer, with rows indexing query positions and columns 
indexing key positions. Dotted gridlines mark the parity-tree level 
boundaries $n_1, n_2, n_3$. Red ticks on the $x$-axis mark the indices 
in the secret set $S \subset [n]$. Each layer's attention concentrates 
sharply on the two children of every parity node at its assigned level.}
    \label{fig:attn-map}
    \vspace{-3mm}
\end{figure}

We empirically verify the results of \Cref{thm:icot-cvg} on the 
$k$-parity task, confirming that a multi-layer transformer trained under 
Log-ICoT converges to the correct parity computation and visualizing how 
reasoning is progressively internalized into deeper layers. We train a 4-layer simplified transformer (\Cref{def:transformer}) 
on the $k$-parity task with $n = 30$ and $k = 16$, giving $L = \log_2 k = 4$ 
training stages. Full details are in 
\Cref{app:imp-details}.

\paragraph{Training dynamics.}
\Cref{fig:mask-loss}(b) shows the validation loss across the four stages of 
Log-ICoT. The first stage trains with full reasoning traces visible, 
after which the loss is already small. At each subsequent stage 
transition, half of the remaining CoT positions become padded, causing 
a transient loss spike as the model internalizes one additional 
parity-tree level into hidden states. The loss then rapidly returns to 
near zero, consistent with \Cref{lemma:base} and 
\Cref{lem:inductive-step}. By the end of stage~4, all CoT positions are 
padded and the model produces the correct parity from input bits alone, 
reaching $100\%$ validation accuracy.

\paragraph{Layer-wise internalization.}
\Cref{fig:attn-map} visualizes the per-layer attention maps at the final 
stage. Each layer's attention concentrates sharply on exactly two key 
positions per query: the two children of the queried parity node in the 
tree decomposition. This 
matches \Cref{lem:inductive-step}(a), which predicts that softmax weights 
at each well-trained layer concentrate on the two children of every 
parity node.

\section{Conclusion}
In this paper, we provide a rigorous theoretical foundation for Implicit 
Chain-of-Thought (ICoT), demonstrating that multi-layer transformers can 
internalize complex reasoning processes within their hidden states without 
sacrificing the sample efficiency gains of explicit CoT. By analyzing the 
parity learning task, we prove that a multi-layer transformer can achieve 
polynomial sample complexity while significantly reducing inference 
latency. Compared to standard ICoT, which removes thinking tokens one at 
a time, Log-ICoT removes them in geometric chunks, reducing the number of 
training stages from linear in $k$ to $\log_2 k$. We discuss our limitations and 
future directions in \Cref{app:sec:limitations}.

\section*{Acknowledgements}

This work was partially supported by a gift from Open Philanthropy to the Center for Human-Compatible AI (CHAI) at UC Berkeley and by NSF Grants IIS-1901252 and CCF-2211209. This work was also supported by NSF grants DMS-2210827, CCF-2315725, CAREER DMS-2339904, ONR grant N00014-24-S-B001, DARPA AIQ grant HR001124S0029-AIQ-FP-003, an Amazon Research Award, a Google Research Scholar Award, an Okawa Foundation Research Grant, and a Sloan Research Fellowship. Y.H. and S.S. were supported by the U.S. Army Research Laboratory and the U.S. Army Research Office under Grant W911NF2010219, Office of Naval Research, and NSF. This work used Jetstream2 at Indiana University through allocation CIS240832 from the Advanced Cyberinfrastructure Coordination Ecosystem: Services \& Support (ACCESS) program, which is supported by National Science Foundation grants \#2138259, \#2138286, \#2138307, \#2137603, and \#2138296.

\bibliography{references}
\bibliographystyle{plainnat}

\newpage
\appendix
\section{Limitations and Future Directions}\label{app:sec:limitations}

Our analysis relies on several simplifications. To make the multi-layer 
training dynamics tractable, we train only the attention matrix while 
fixing the other components: identity value matrices, a link function 
in place of MLPs, and prescribed output matrices and gating weights. Making the gates input-dependent or trainable, as in \citet{parisotto2020stabilizing, 
zhu2024hyper}, is a natural next step. Our experiments are also limited to 
the synthetic parity task and a four-layer transformer. Although \citet{deng2024explicit} demonstrate the effectiveness of
ICoT on practical reasoning tasks, implementing Log-ICoT in LLMs would require a heuristic for dividing the curriculum 
into stages, which is non-trivial without the explicit hierarchical 
structure that parity provides. Finally, we view connecting ICoT to 
self-distillation \citep{zhao2026self, hubotter2026reinforcement} as a 
promising direction. Empirical work \citep{hubotter2026reinforcement} in that line observes that reasoning 
traces shorten substantially after training, suggesting that a similar 
internalization mechanism may be at play.

\section{Additional Details}\label{app:imp-details}
\begin{figure}[t]
    \centering
    \subfigure{\includegraphics[width=0.95\textwidth]{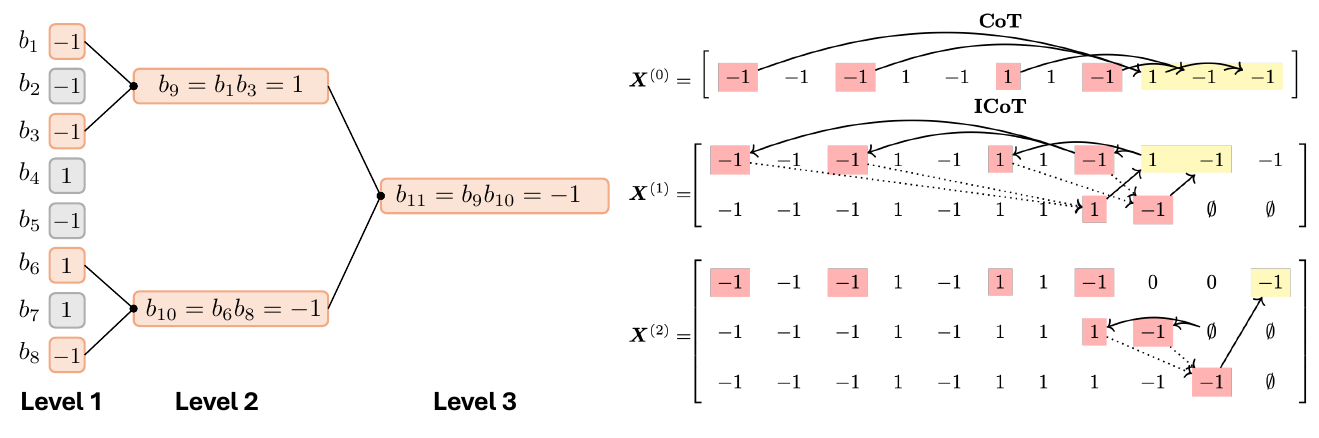}}
    \caption{\textbf{Illustration of parity learning task with input length $n = 8$ and secret set size $k =4$.} 
    \textit{Left:} The task can be decomposed into a hierarchical two-parity computation.
    \textit{Right:} Comparison of the training curriculum of Chain-of-Thought (CoT) and Implicit CoT (ICoT). 
    Both methods initially leverage complete thinking traces derived from the hierarchical decomposition. As the ICoT curriculum progresses, these intermediate reasoning steps are replaced by padding tokens ($0$), forcing the model to internalize the computation within its hidden states.
    }
    \label{fig:tree-diagram}
\end{figure}

\subsection{Example of the Parity-Learning Decomposition}\label{app:sec:parity}

To complement the abstract definitions in \Cref{sec:task-defn} and the 
proof sketch in \Cref{sec:proof-sketch}, \Cref{fig:tree-diagram} provides 
a concrete worked example with $n = 8$ input bits and secret set size 
$k = 4$. The left panel illustrates the hierarchical decomposition of 
the $k$-parity task into a binary tree of two-parity computations: each 
internal node $b_m$ ($m > n$) computes the parity of its two children, 
with the root $b_{11}$ recovering the target label $y$. The right panel 
illustrates the residual-stream evolution $\X^{(0)}, \X^{(1)}, \X^{(2)}$ 
across layers, contrasting explicit CoT (top) with the ICoT curriculum 
(bottom), where intermediate reasoning positions are progressively replaced 
by padding tokens ($0$) and the model must internalize the 
computation in its hidden states.

\subsection{Implementation Details}\label{app:experiment-details}                                                    We provide additional implementation details of \Cref{sec:exp}.  
  \paragraph{Link function.}                                                               
  Following \Cref{def:link}, we require a smooth $\phi$ satisfying
  $\phi(0) = -1$, $\phi(\pm 1) = 1$, and                                                                                                                      
  $\phi'(0) = \phi'(\pm 1) = 0$. We use the cosine link                                                                                                       
  $\phi(t) = -\cos(\pi t)$, which satisfies all three conditions exactly:                                                                                     
  $\phi'(t) = \pi \sin(\pi t)$ vanishes at $t \in \{0, \pm 1\}$. Its local                                                                                    
  Taylor expansion around the origin,                                                                                                                         
  $\phi(t) = -1 + \tfrac{\pi^2}{2} t^2 + O(t^4)$, matches the form                                                                                            
  $\phi(t) = -1 + c t^2 + O(t^4)$ used in our analysis with constant                                                                                          
  $c = \pi^2 / 2$.  
  
\paragraph{Data generation.} 
Inputs are sampled uniformly from $\{\pm 1\}^n$ with $n = 30$. The secret 
set $S \subset [n]$ with $|S| = k = 16$ is sampled uniformly from all 
size-$k$ subsets of $[n]$ at the start of training and held fixed thereafter. 
The label $y = \prod_{j \in S} b_j$ is computed deterministically from 
the input. CoT positions $b_{n+1}, \ldots, b_T$ are computed via the 
recursion $b_m = b_{c_1[m]} \cdot b_{c_2[m]}$ on the parity tree.

\paragraph{Training and evaluation.} 
{We optimize with AdamW at learning 
rate $10^{-1}$ and batch size $500$, run $500$ steps per stage for a total 
of $2{,}000$ training steps, and evaluate on $2{,}000$ held-out samples 
every $25$ steps.}

\section{Proof of \Cref{sec:theory}}  
\label{app:sec:theory}
\paragraph{Notations. }
Recall from \Cref{sec:task-defn} that $b_1, \ldots, b_n$ are i.i.d. 
uniform on $\{\pm 1\}^n$ and that $b_{n+1}, \ldots, b_T$ are determined 
by the parity-tree recursion based on the secret set $S$. For a tuple of indices $(j_1, \dots, j_r) \in [T]^r$, we say a parity product $b_{j_1} \cdots 
b_{j_r}$ is \emph{trivial} if it always equals $1$. E.g., the parity $b_1 b_3 b_9$ is trivial in \Cref{fig:tree-diagram}. We define the set of non-trivial tuples of length $r$ over the first $m$ indices as:
\[
    I_{r, m}  = \{(j_1, \dots, j_r) \in [m]^r , b_{j_1} \cdots b_{j_r} \not \equiv 1 \}.
\]
For $r = 1$, every index is non-trivial so $I_{1,m} = [m]$.
As defined in \Cref{sec:model}, $\Psi_{t'}^\top \mOutput^{(t)} = \one{t = t '} \begin{bmatrix}
        \bzero_{B \times T} , \eb_{L + 1, t}^\top \otimes \Ib_B
    \end{bmatrix} \in \R^{B \times d}$. Let $\Sb_t = \eb_{L + 1, t}^\top \otimes \Ib_B \in \R^{B \times B_L}$, we denote $\Pb_t = \Psi_{t}^\top \mOutput^{(t)} = \begin{bmatrix}
        \bzero_{B \times T} , \Sb_t
    \end{bmatrix}$, which extracts the $t$-th $B$-dimensional block.
    For any $\x \in \R^{d}$, we define the $t$-th $B$-block (after the first $T$ coordinates) by
    \[
        \x_{[t]} := S_t \x_{T:} \in \R^{B}.
    \]
    Equivalently, we have $(\x_{[t]})_i = x_{T + B(t - 1) + i}$. Finally, for $x \in \R$ and $S \subseteq \R$, we write 
$\dist{x, S} := \inf_{s \in S} |x - s|$.

\subsection{Proof of \Cref{lemma:constant-input}} \label{app:sec:proof-constant-input}
\begin{proof}
    Fix any $n^{\epsilon / 8} \leq m \leq n$.
    Since $\bb_\alpha \sim \unif(\{\pm 1\}^B)$ for $\alpha \in [m]$. By Hoeffding's inequality, we have that:
    \[
       Pr ( \|\frac{1}{m - 1}\sum_{\alpha \in [m - 1]} \bb_\alpha \|_\infty \geq t) \leq 2B \exp (-\frac{mt^2}{2}).
    \]
    Thus with probability $1 - p_1$, we have:
    \[
        \|\frac{1}{m - 1}\sum_{\alpha \in [m - 1]} \bb_\alpha \|_\infty \leq \sqrt{\frac{2}{m} \log (2B / p_1)}.
    \]
    Since $\phi$ behaves like a quadratic near $0, \pm 1$, there exists some constant $C_2 > 0$ that depends only on $\phi$ such that
    \begin{align*}
        \|\phi(\frac{1}{m - 1}\sum_{\alpha \in [m - 1]} \bb_\alpha ) + \onebb_B \|_\infty = \|\phi(\frac{1}{m - 1}\sum_{\alpha \in [m - 1]} \bb_\alpha ) -\phi(\bzero_B) \|_\infty 
        & \leq C_2 \left(\sqrt{\frac{2}{m} \log (2B / p_1)} \right)^2 
        \\ & \leq C_2 \frac{2}{m} \log (2B / p_1).
    \end{align*}
    {Finally, assume the failure probability for each $m$ be $\delta_m = p_1 / T$. With probability $1 - p_1$, for all $n^{\epsilon / 8} \leq m \leq T < 2n$}, we have
    \[
        \|\phi(\frac{1}{m - 1}\sum_{\alpha \in [m - 1]} \bb_\alpha ) + \onebb_B \|_\infty \leq C_2 \frac{2}{m} \log (2TB / p_1).
    \]
    Recall $B = \poly (n)$ and set $p_1 =\exp(-n^{\epsilon / 16}) $, we have
    \[
        \|\phi(\frac{1}{m - 1}\sum_{\alpha \in [m - 1]} \bb_\alpha ) + \onebb_B \|_\infty \leq O(n^{- \epsilon / 16}).
    \]
\end{proof}

\subsection{Residual stream and gradient propagation}
\begin{lemma}[Gated Connection Propagation]\label{lemma:hc-propagation}
Using the gating weights in Eq.~\eqref{eq:gating-weights}, for any position $\beta \in [T]$ with $\nu := h[\beta]$. The 
residual stream satisfies:
\[
\x^{(\ell)}_{\beta, [\ell + 1]} = 
\begin{cases}
\x^{(0)}_{\beta, [1]}, & \nu = 1 \text{ or } 0 \leq \ell < \nu - 1, \\
\phi(\hat \z^{(\nu - 1)}_{\beta, [\nu - 1]}), & \nu \geq 2 \text{ and } \ell \geq \nu - 1.
\end{cases}
\]
\end{lemma}
\begin{proof}
Recall the gated connection update at layer $\ell$:
\[
\x^{(\ell)}_{\beta} = \x^{(\ell - 1)}_{\beta} + \mOutput^{(\ell)} \cdot 
\hc^{(\ell)}_{\beta},
\]
where $\hc^{(\ell)}_{\beta} = g^{(\ell)}_\beta \phi(\hat \z^{(\ell)}_{\beta}) 
+ (1 - g^{(\ell)}_\beta) \x^{(\ell - 1)}_{\beta}$ and $\mOutput^{(\ell)}$ 
is the block-shift operator that places content from residual block $[\ell]$ to block $[\ell + 1]$.

Componentwise, the block-$[\ell + 1]$ update at position $\beta$ reads
\begin{equation}\label{eq:hc-update}
\x^{(\ell)}_{\beta, [\ell + 1]} = \x^{(\ell - 1)}_{\beta, [\ell + 1]} 
+ g^{(\ell)}_\beta \phi(\hat \z^{(\ell)}_{\beta, [\ell]}) 
+ (1 - g^{(\ell)}_\beta) \x^{(\ell - 1)}_{\beta, [\ell]},
\tag{$\star$} 
\end{equation}
while blocks $[k]$ for $k \neq \ell + 1$ are preserved, i.e., 
$\x^{(\ell)}_{\beta, [k]} = \x^{(\ell - 1)}_{\beta, [k]}$ for $k \neq \ell + 1$.

Recall the per-block gating from Eq.~\eqref{eq:gating-weights}:
\[
g^{(\ell)}_\beta = 
\begin{cases}
1, & n_\ell + 1 \leq \beta \leq n_{\ell + 1}, \\
0, & \text{otherwise}.
\end{cases}
\]
Hence $g^{(\ell)}_\beta = 1$ iff $\ell = \nu - 1$ (the writing layer for 
position $\beta$); otherwise $g^{(\ell)}_\beta = 0$.

\paragraph{Case 1: $\nu = 1$ (input position).}
For $\beta \leq n$, $g^{(\ell)}_\beta = 0$ for all $\ell$, so 
Eq.~\eqref{eq:hc-update} reduces to
\[
\x^{(\ell)}_{\beta, [\ell + 1]} = \x^{(\ell - 1)}_{\beta, [\ell + 1]} 
+ \x^{(\ell - 1)}_{\beta, [\ell]}.
\]
We prove by induction on $\ell$ that $\x^{(\ell)}_{\beta, [\ell + 1]} = 
\x^{(0)}_{\beta, [1]}$ for all $\ell \geq 0$. The base case $\ell = 0$ holds 
trivially. For the inductive step, assume 
$\x^{(\ell - 1)}_{\beta, [\ell]} = \x^{(0)}_{\beta, [1]}$. Block $[\ell + 1]$ 
is unchanged at all earlier layers (since the writing layer for block 
$[\ell + 1]$ is layer $\ell$, and earlier layers don't touch block $[\ell + 1]$). 
Hence $\x^{(\ell - 1)}_{\beta, [\ell + 1]} = \x^{(0)}_{\beta, [\ell + 1]} = \bzero_B$ 
(\Cref{sec:model}, the embedding has only block $[1]$ nonzero). 
Therefore
\[
\x^{(\ell)}_{\beta, [\ell + 1]} = \bzero_B + \x^{(0)}_{\beta, [1]} = \x^{(0)}_{\beta, [1]}.
\]

\paragraph{Case 2: $\nu \geq 2$ and $0 \leq \ell < \nu - 1$ (before writing layer).}
For these layers, $g^{(\ell)}_\beta = 0$, so Eq.~\eqref{eq:hc-update} again 
reduces to
\[
\x^{(\ell)}_{\beta, [\ell + 1]} = \x^{(\ell - 1)}_{\beta, [\ell + 1]} 
+ \x^{(\ell - 1)}_{\beta, [\ell]}.
\]
By the same induction on $\ell$ as in Case 1, 
$\x^{(\ell)}_{\beta, [\ell + 1]} = \x^{(0)}_{\beta, [1]}$ for $0 \leq \ell < \nu - 1$. 
Note that $\x^{(0)}_{\beta, [1]}$ here is the embedding content at position 
$\beta$, which can be $\bb_\beta$ (shown CoT) or $\bzero_B$ (padded CoT) 
depending on the curriculum's stage.

\paragraph{Case 3: $\nu \geq 2$ and $\ell = \nu - 1$ (writing layer).}
At the writing layer, $g^{(\nu - 1)}_\beta = 1$, so 
Eq.~\eqref{eq:hc-update} becomes
\[
\x^{(\nu - 1)}_{\beta, [\nu]} = \x^{(\nu - 2)}_{\beta, [\nu]} 
+ \phi(\hat \z^{(\nu - 1)}_{\beta, [\nu - 1]}).
\]
Block $[\nu]$ has not been written by any earlier layer (the writing layer 
for block $[\nu]$ is layer $\nu - 1$, the current layer). Hence 
$\x^{(\nu - 2)}_{\beta, [\nu]} = \x^{(0)}_{\beta, [\nu]} = \bzero_B$. Therefore
\[
\x^{(\nu - 1)}_{\beta, [\nu]} = \phi(\hat \z^{(\nu - 1)}_{\beta, [\nu - 1]}).
\]

\paragraph{Case 4: $\nu \geq 2$ and $\ell \geq \nu$ (after writing layer).}
For $\ell \geq \nu$, $g^{(\ell)}_\beta = 0$ (since $\beta \leq n_{\ell - 1} 
< n_\ell + 1$), so Eq.~\eqref{eq:hc-update} reduces to
\[
\x^{(\ell)}_{\beta, [\ell + 1]} = \x^{(\ell - 1)}_{\beta, [\ell + 1]} 
+ \x^{(\ell - 1)}_{\beta, [\ell]}.
\]
We prove by induction on $\ell \geq \nu$ that 
$\x^{(\ell)}_{\beta, [\ell + 1]} = \phi(\hat \z^{(\nu - 1)}_{\beta, [\nu - 1]})$. 
The base case $\ell = \nu - 1$ is Case 3.

For $\ell \geq \nu$: by the inductive hypothesis applied at $\ell - 1 \geq \nu - 1$, 
$\x^{(\ell - 1)}_{\beta, [\ell]} = \phi(\hat \z^{(\nu - 1)}_{\beta, [\nu - 1]})$. 
Block $[\ell + 1]$ has not been written at any earlier layer, so 
$\x^{(\ell - 1)}_{\beta, [\ell + 1]} = \bzero_B$. Therefore
\[
\x^{(\ell)}_{\beta, [\ell + 1]} = \bzero_B + \phi(\hat \z^{(\nu - 1)}_{\beta, [\nu - 1]}) 
= \phi(\hat \z^{(\nu - 1)}_{\beta, [\nu - 1]}).
\]

Combining Cases 1, 2 (which give $\x^{(0)}_{\beta, [1]}$) and Cases 3, 4 
(which give $\phi(\hat \z^{(\nu - 1)}_{\beta, [\nu - 1]})$), the lemma follows.
\end{proof}

\begin{corollary}[Frozen Derivative]\label{cor:frozen-derivative}
Under the setting of \Cref{lemma:hc-propagation}, for any $\beta \in [T]$ 
with $\nu := h[\beta]$, any layer $\ell \in \{0, \ldots, L\}$, and any 
parameter $w^{(\ell')}_{j, m}$ with $\ell' \in [L]$:
\[
\frac{\partial \x^{(\ell)}_{\beta, [\ell + 1]}}{\partial w^{(\ell')}_{j, m}} = 
\begin{cases}
\dfrac{\partial \phi(\hat \z^{(\nu - 1)}_{\beta, [\nu - 1]})}{\partial w^{(\ell')}_{j, m}}, 
& \nu \geq 2, \, \ell \geq \nu - 1, \, \ell' \leq \nu - 1, \\[8pt]
\bzero, & \text{otherwise}.
\end{cases}
\]
\end{corollary}
\begin{proof}
By \Cref{lemma:hc-propagation}, $\x^{(\ell)}_{\beta, [\ell + 1]}$ takes 
two forms:
\begin{itemize}
\item If $\nu = 1$, or if $\nu \geq 2$ with $\ell < \nu - 1$: 
$\x^{(\ell)}_{\beta, [\ell + 1]} = \x^{(0)}_{\beta, [1]}$. The embedding 
is data, not a trainable parameter, so the derivative vanishes.

\item If $\nu \geq 2$ and $\ell \geq \nu - 1$: 
$\x^{(\ell)}_{\beta, [\ell + 1]} = \phi(\hat \z^{(\nu - 1)}_{\beta, [\nu - 1]})$. 
By chain rule,
\[
\frac{\partial \x^{(\ell)}_{\beta, [\ell + 1]}}{\partial w^{(\ell')}_{j, m}}
= \frac{\partial \phi(\hat \z^{(\nu - 1)}_{\beta, [\nu - 1]})}{\partial w^{(\ell')}_{j, m}}.
\]
Since $\phi(\hat \z^{(\nu - 1)}_{\beta, [\nu - 1]})$ depends only on 
$\{\W^{(\tau)}\}_{\tau \leq \nu - 1}$, the derivative vanishes when 
$\ell' > \nu - 1$, and equals the chain-rule expression when $\ell' \leq \nu - 1$.
\end{itemize}

In all sub-cases, the derivative matches the corollary's statement.
\end{proof}

\begin{lemma}[Per-Node Sensitivity] \label{lem:per-node-sensitivity}
Let $L_\phi := \|\phi'\|_\infty$. For any layer indices $\ell' \leq \ell \leq L$, 
{any query position $m \in [n_{\ell'} + 1, n_{\ell' + 1}]$, and any 
$1 \leq j \leq m - 1$}, define
\[
\xi^{(\ell)}_{j, m, \ell'} := \max_{\beta : h[\beta] \leq \ell + 1} 
\left\| \frac{\partial \x^{(\ell)}_{\beta, [\ell+1]}}{\partial w^{(\ell')}_{j, m}} \right\|_\infty.
\]
Then
\[
\xi^{(\ell)}_{j, m, \ell'} \leq 2 L_\phi^{\ell - \ell' + 1} \sigma_j(\w^{(\ell')}_m).
\]
\end{lemma}
\begin{proof}
We induct on $\ell$.

\paragraph{Base case $\ell = \ell'$.}
By \Cref{cor:frozen-derivative}, the derivative is nonzero only when 
$\nu \geq 2$, $\ell' \geq \nu - 1$, and $\ell' \leq \nu - 1$, i.e., 
$\nu - 1 = \ell'$, so $\nu = \ell' + 1$. {For positions with 
$\nu > \ell' + 1$, the derivative vanishes since the residual equals the embedding which is not a trainable parameter.} For positions with $\nu \leq \ell'$, 
the writing layer is at most $\ell' - 1 < \ell'$, so the derivative also vanishes by \Cref{cor:frozen-derivative}. 

Hence only $\beta$ with $\nu = \ell' + 1$ contribute. For such $\beta$, 
\Cref{cor:frozen-derivative} yields
\begin{align*}
    \frac{\partial \x^{(\ell')}_{\beta, [\ell' + 1]}}{\partial w^{(\ell')}_{j, m}} 
    = \frac{\partial \phi(\hat \z^{(\ell')}_{\beta, [\ell']})}{\partial w^{(\ell')}_{j, m}} 
    = \phi'(\hat \z^{(\ell')}_{\beta, [\ell']}) \odot 
    \frac{\partial \hat \z^{(\ell')}_{\beta, [\ell']}}{\partial w^{(\ell')}_{j, m}}.
\end{align*}
The parameter $w^{(\ell')}_{j, m}$ enters $\hat \z^{(\ell')}_\beta$ only 
through the softmax weights at query position $\beta = m$, giving
\begin{align*}
    \frac{\partial \hat \z^{(\ell')}_{m, [\ell']}}{\partial w^{(\ell')}_{j, m}} 
    = \sigma_j(\w^{(\ell')}_m) \bigl( \x^{(\ell' - 1)}_{j, [\ell']} - \hat \z^{(\ell')}_{m, [\ell']} \bigr).
\end{align*}
Each component of $\x^{(\ell' - 1)}_{j, [\ell']} - \hat \z^{(\ell')}_{m, [\ell']}$ 
lies in $[-2, 2]$, so
\begin{align*}
    \xi^{(\ell')}_{j, m, \ell'} \leq L_\phi \cdot 2 \sigma_j(\w^{(\ell')}_m) 
    = 2 L_\phi \sigma_j(\w^{(\ell')}_m),
\end{align*}
matching the claim with $\ell - \ell' + 1 = 1$.

\paragraph{Inductive step $\ell \Rightarrow \ell + 1$.}
Take any $\beta \in [T]$ with $\nu := h[\beta]$. By \Cref{cor:frozen-derivative}, 
the derivative is zero unless $\nu \geq \ell' + 1$ {and $\ell + 1 \geq \nu - 1$, 
i.e., $\nu \leq \ell + 2$}. So $\nu \in \{\ell' + 1, \ldots, \ell + 2\}$. 

We split into two sub-cases.

\emph{Sub-case A: $\nu \leq \ell + 1$.} By \Cref{lemma:hc-propagation}, both 
$\x^{(\ell + 1)}_{\beta, [\ell + 2]}$ and $\x^{(\ell)}_{\beta, [\ell + 1]}$ 
equal the same value $\phi(\hat \z^{(\nu - 1)}_{\beta, [\nu - 1]})$, hence 
have identical derivatives:
\begin{align*}
    \left\| \frac{\partial \x^{(\ell + 1)}_{\beta, [\ell + 2]}}{\partial w^{(\ell')}_{j, m}} \right\|_\infty 
    = \left\| \frac{\partial \x^{(\ell)}_{\beta, [\ell + 1]}}{\partial w^{(\ell')}_{j, m}} \right\|_\infty 
    \leq \xi^{(\ell)}_{j, m, \ell'}.
\end{align*}

\emph{Sub-case B: $\nu = \ell + 2$.} Then $\beta$'s writing layer is 
$\ell + 1$, and \Cref{cor:frozen-derivative} gives
\begin{align*}
    \frac{\partial \x^{(\ell + 1)}_{\beta, [\ell + 2]}}{\partial w^{(\ell')}_{j, m}} 
    = \frac{\partial \phi(\hat \z^{(\ell + 1)}_{\beta, [\ell + 1]})}{\partial w^{(\ell')}_{j, m}} 
    = \phi'(\hat \z^{(\ell + 1)}_{\beta, [\ell + 1]}) \odot 
    \frac{\partial \hat \z^{(\ell + 1)}_{\beta, [\ell + 1]}}{\partial w^{(\ell')}_{j, m}}.
\end{align*}
Since $\ell' \leq \ell < \ell + 1$, the parameter $w^{(\ell')}_{j, m}$ does 
not enter the softmax weights $\sigma_\cdot(\w^{(\ell + 1)}_\beta)$, so 
differentiation passes through:
\begin{align*}
    \frac{\partial \hat \z^{(\ell + 1)}_{\beta, [\ell + 1]}}{\partial w^{(\ell')}_{j, m}} 
    = \sum_{\gamma = 1}^{{\beta - 1}} \sigma_\gamma(\w^{(\ell + 1)}_\beta) \, 
    \frac{\partial \x^{(\ell)}_{\gamma, [\ell + 1]}}{\partial w^{(\ell')}_{j, m}}.
\end{align*}
Each summand has $\infty$-norm at most $\xi^{(\ell)}_{j, m, \ell'}$, and the 
softmax weights sum to at most $1$, so
\begin{align*}
    \left\| \frac{\partial \hat \z^{(\ell + 1)}_{\beta, [\ell + 1]}}{\partial w^{(\ell')}_{j, m}} \right\|_\infty 
    \leq \xi^{(\ell)}_{j, m, \ell'}.
\end{align*}
Multiplying by $\|\phi'\|_\infty = L_\phi$:
\begin{align*}
    \left\| \frac{\partial \x^{(\ell + 1)}_{\beta, [\ell + 2]}}{\partial w^{(\ell')}_{j, m}} \right\|_\infty 
    \leq L_\phi \xi^{(\ell)}_{j, m, \ell'}.
\end{align*}

\emph{Combining sub-cases.} Taking the max over $\beta$ and using $L_\phi \geq 1$,
\begin{align*}
    \xi^{(\ell + 1)}_{j, m, \ell'} \leq L_\phi \xi^{(\ell)}_{j, m, \ell'} 
    \leq L_\phi \cdot 2 L_\phi^{\ell - \ell' + 1} \sigma_j(\w^{(\ell')}_m) 
    = 2 L_\phi^{(\ell + 1) - \ell' + 1} \sigma_j(\w^{(\ell')}_m).
\end{align*}
This closes the induction.
\end{proof}

\subsection{Bounded Gradient} \label{app:bounded-gd}
\begin{lemma}[Bounded Gradient]
\label{lemma:bounded-gd}
For any stage $t \in [L]$, any input $\bb \in \{\pm 1\}^n$ and any 
parameters $\bte$,
\[
\| \nabla_\bte f_\bte^{(t)}(\bb) \|_2 \leq O(n^g),
\]
where $f_\theta^{(t)}(\bb) := \Psi_t^\top \mathcal{T}_\theta(\bb) \in \mathbb{R}^T$ 
is the stage-$t$ model output, $g = \log_2 L_\phi + 1$ and $L_\phi := \|\phi'\|_\infty$.
\end{lemma}

\begin{proof}
Fix a stage $t \in [L]$. The stage-$t$ model output reads $f_\theta^{(t)}(\bb) 
= \Psi_t^\top \mathcal{T}_\theta(\bb)$, where $\Psi_t$ extracts block $[t+1]$. 
By \Cref{lemma:hc-propagation}, for any $\alpha \in [T]$ with $h[\alpha] \le t + 1$, 
the readout coordinate is
\[
f^{(t)}_{\theta, \alpha}(\bb) := \bigl(\Psi_t^\top \mathcal{T}_\theta(\bb)\bigr)_\alpha 
= \x^{(t)}_{\alpha, [t+1]},
\]
and for $\alpha$ with $h[\alpha] > t + 1$, the readout reduces to the 
embedding due to \Cref{lemma:hc-propagation} (which has zero gradient). It therefore suffices to bound 
$\|\nabla_\theta f^{(t)}_{\theta, \alpha}(\bb)\|$ at each $\alpha$ with 
$h[\alpha] \le t + 1$, then sum. Fix such an $\alpha$.

\paragraph{Bounding per-parameter derivatives.}
For any trainable parameter $w^{(\ell')}_{j, m}$ with $\ell' \in [L]$, 
{$n_{\ell'} + 1 \le m \le n_{\ell' + 1}$, and $1 \le j \le m - 1$}, 
\Cref{lem:per-node-sensitivity} yields
\[
\left| \frac{\partial f^{(t)}_{\theta, \alpha}(\bb)}{\partial w^{(\ell')}_{j, m}} \right| 
= \left| \frac{\partial \x^{(t)}_{\alpha, [t+1]}}{\partial w^{(\ell')}_{j, m}} \right| 
\le \xi^{(t)}_{j, m, \ell'} \le 2 L_\phi^{t - \ell' + 1} \sigma_j(\w^{(\ell')}_m).
\]
The bound is non-trivial only for $\ell' \le t$; for $\ell' > t$, the 
derivative vanishes by \Cref{cor:frozen-derivative} since $\alpha$'s 
writing layer $h[\alpha] - 1 \le t < \ell'$.

\paragraph{Summing over parameters.}
We compute $\|\nabla_\theta f^{(t)}_{\theta, \alpha}(\bb)\|^2$ by summing 
the squared derivatives over all trainable parameters. Square the 
per-parameter bound:
\[
\left| \frac{\partial f^{(t)}_{\theta, \alpha}(\bb)}{\partial w^{(\ell')}_{j, m}} \right|^2 
\le 4 L_\phi^{2(t - \ell' + 1)} \sigma_j(\w^{(\ell')}_m)^2.
\]

\emph{Sum over $j$.} Using {$\sum_{j=1}^{m - 1} \sigma_j(\w^{(\ell')}_m)^2 
\le \sum_{j=1}^{m-1} \sigma_j(\w^{(\ell')}_m) = 1$},
\[
\sum_{j=1}^{{m - 1}} \left| \frac{\partial f^{(t)}_{\theta, \alpha}(\bb)}{\partial w^{(\ell')}_{j, m}} \right|^2 
\le 4 L_\phi^{2(t - \ell' + 1)}.
\]

\emph{Sum over $m$.} {Trainable indices $m$ at layer $\ell'$ satisfy 
$n_{\ell'} + 1 \le m \le n_{\ell' + 1}$, with $n_{\ell' + 1} - n_{\ell'} = k \cdot 2^{-\ell'}= O(n)$,} so
\[
\sum_{m, j} \left| \frac{\partial f^{(t)}_{\theta, \alpha}(\bb)}{\partial w^{(\ell')}_{j, m}} \right|^2 
\le O(n) \cdot L_\phi^{2(t - \ell' + 1)}.
\]

\emph{Sum over $\ell' \le t$.} Letting $r := t - \ell' + 1$ range over 
$1, 2, \ldots, t$,
\[
\sum_{\ell' = 1}^{t} \sum_{m, j} \left| \frac{\partial f^{(t)}_{\theta, \alpha}(\bb)}{\partial w^{(\ell')}_{j, m}} \right|^2 
\le \sum_{r = 1}^{t} O(n) \cdot L_\phi^{2r} = O(n L_\phi^{2t}).
\]
Since $t \le L = \log_2 k$ and $k = O(n)$,
\[
O(n L_\phi^{2t}) \le O(n L_\phi^{2L}) = O(n \cdot k^{2 \log_2 L_\phi}) 
= O(n^{2 \log_2 L_\phi + 1}).
\]
Therefore,
\[
\| \nabla_\theta f^{(t)}_{\theta, \alpha}(\bb) \|^2 \le O(n^{2 \log_2 L_\phi + 1}),
\]
which gives $\| \nabla_\theta f^{(t)}_{\theta, \alpha}(\bb) \| 
\le O(n^{\log_2 L_\phi + 1/2}) = O(n^{g - 1/2})$.

\paragraph{Conclusion.}
Summing over $\alpha \in [T]$ with $T = O(n)$,
\[
\| \nabla_\theta f^{(t)}_\theta(\bb) \|^2 = \sum_\alpha \| \nabla_\theta f^{(t)}_{\theta, \alpha}(\bb) \|^2 
\le O(n) \cdot O(n^{2g - 1}) = O(n^{2g}),
\]
giving $\| \nabla_\theta f^{(t)}_\theta(\bb) \| \le O(n^g)$, as claimed.
\end{proof}

\subsection{Preliminary Lemmas}
\begin{lemma}[\citet{kim2024transformers}, Lemma 9] \label{lemma:trivial-noise} Recall that $I_{r, m}$ is defined as the set of non-trivial tuples of length $r$ over the first $m$ indices as:
\[
    I_{r, m}  = \{(j_1, \dots, j_r) \in [m]^r , b_{j_1} \cdots b_{j_r} \not \equiv 1 \}.
\] Assume $\bb_1, \dots, \bb_m \in \{\pm 1 \}^B $ are vectors where each bit is sampled i.i.d. from the uniform distribution. For any $p > 0$, with probability at least $1 -p$, the following holds for all $r \leq 4$:
    \[
       \max_{(j_1, \dots, j_r) \in I_{r, m}} \frac{|\langle \bb_{j_1}, \dots, \bb_{j_r} \rangle | }{B} \leq \kappa := \sqrt{\frac{2}{B} \log \frac{32n ^4}{p}}.
    \]
\end{lemma}
In particular, we set $B = \Omega(n^{2 + \epsilon})$ for some constant $\epsilon > 0$ and $p = \exp(-n^{\epsilon / 2})$ such that $\kappa = O(n^{-1 - \epsilon / 4})$. \Cref{lemma:trivial-noise} bounds the contribution of \emph{non-trivial} tuples 
(those in $I_{r,m}$), where each inner product $|\langle \bb_{j_1}, \dots, \bb_{j_r}\rangle|/B$ is at most $\kappa$. 
The complementary contribution comes from \emph{trivial} tuples 
(those in $[m]^r \setminus I_{r,m}$), where the inner product equals $B$ 
deterministically. Bounding such gradient contractions therefore reduces 
to counting trivial tuples, which we do next for the case $r = 4$.
\begin{lemma}[Trivial 4-tuple count]\label{lem:trivial-4-tuples}
Let $\bb_1, \dots, \bb_n \in \{\pm 1\}^B$ be i.i.d. uniform input bits, and 
for $\tau \in [n+1, T]$ let $\bb_\tau := \bb_{c_1[\tau]} \odot \bb_{c_2[\tau]}$ 
denote the parity defined by the tree (\Cref{sec:task-defn}). For any $m \in [n+1, T]$, the number of 4-tuples 
$(\alpha, \beta, \gamma, \delta) \in [m-1]^4$ with 
$\langle \bb_\alpha, \bb_\beta, \bb_\gamma, \bb_\delta\rangle$ trivial 
(i.e., $(\alpha, \beta, \gamma, \delta) \notin I_{4, m}$) is at most $O(n^2)$.
\end{lemma}
\begin{proof}
By the bit-counting argument, $\langle \bb_\alpha, \bb_\beta, \bb_\gamma, \bb_\delta\rangle \equiv B $ happens iff the multiset of input bits in 
the parity product cancels evenly. Without loss of generality, assume 
$h[\alpha] \leq h[\beta] \leq h[\gamma] \leq h[\delta]$. Under the causal mask, 
$\alpha, \beta, \gamma, \delta \in [m - 1]$ may take any level, so we enumerate 
by the level structure.
\begin{enumerate}
    \item If $h[\beta] < h[\gamma] < h[\delta]$, to cancel out $\bb_\delta$, 
    it must hold that $\bb_\gamma$ is a child of $\bb_\delta$ and 
    $\bb_\alpha, \bb_\beta$ are the two children of the other child. This is 
    fully determined by choosing $\delta$ and which child is $\gamma$, giving 
    $O(n)$ trivial 4-tuples.
    \item If $h[\beta] = h[\gamma] < h[\delta]$, then $\bb_\beta, \bb_\gamma$ 
    must both be children of $\bb_\delta$, requiring $h[\gamma] = h[\delta] - 1$. 
    Then $\langle \bb_\beta, \bb_\gamma, \bb_\delta\rangle = B$, and for 
    $\langle \bb_\alpha, \bb_\beta, \bb_\gamma, \bb_\delta\rangle$ to be 
    trivial, we would need $\langle \bb_\alpha\rangle = B$, which is 
    impossible. No trivial 4-tuples.
    \item If $h[\beta] < h[\gamma] = h[\delta]$, then $\gamma = \delta$ 
    (since same-level peers cover disjoint bits, two distinct peers cannot 
    have trivial product). This forces $\alpha = \beta$ similarly. There are 
    $O(n^2)$ such trivial 4-tuples.
    \item If $h[\alpha] = h[\beta] = h[\gamma] = h[\delta]$, then the four 
    same-level nodes must pair up to cancel. There are $O(n^2)$ such trivial 4-tuples.
    \item If $h[\alpha] < h[\beta] = h[\gamma] = h[\delta]$, then 
    $\langle \bb_\beta, \bb_\gamma, \bb_\delta\rangle$ involves 
    at least $2^{h[\beta] - 1}$ distinct input bits, so $\bb_\alpha$ would need to cover all these bits for the 
    4-tuple to be trivial. But $\bb_\alpha$ is a parity over only 
    $2^{h[\alpha] - 1} < 2^{h[\beta] - 1}$ bits. No trivial 4-tuples.
\end{enumerate}
Combining cases 1, 3, and 4, the total count of trivial 4-tuples is 
$O(n) + O(n^2) + O(n^2) = O(n^2)$. Cases 2 and 5 contribute none.
\end{proof}

We now proceed to the main proof of \Cref{thm:icot-cvg}.
\subsection{Proof of \Cref{thm:icot-cvg}}
\begin{definition}[Well-trained layer]
\label{def:well-trained}
Layer $\ell \in [L]$ is \emph{well-trained} if:
\begin{itemize}
\item (softmax concentration) For all $m \in [n_\ell + 1, n_{\ell+1}]$,
\[
\frac{1 - \exp(-Cn^{\epsilon/16})}{2} 
\leq \sigma_{c_1[m]}(\w_m^{(\ell)}), \sigma_{c_2[m]}(\w_m^{(\ell)}) \leq \frac{1}{2};
\]
\item (forward error) For all $m \in [n_\ell + 1, n_{\ell+1}]$,
$\|\x_{m,[\ell+1]}^{(\ell)} - \bb_m\|_\infty \leq \exp(-Cn^{\epsilon/16})$.
\end{itemize}
\end{definition}
\begin{lemma}[Stage 1]\label{lemma:base}
After one-step gradient descent on the stage-1 loss 
$\Lc^{(1)}$ with $B$ training samples and stage-wise learning rate $\eta_1 = \frac{K_n n^2}{2c}$, where $B = \Omega(n^{2+\epsilon})$, 
$K_n := \lceil n^{\epsilon/16} \rceil$ and \(c>0\) is the link function constant in \Cref{def:link}, with probability 
$1 - \exp(-n^{\epsilon/2})$ over the sampling of $\bb^1, \ldots, \bb^B$, 
layer 1 is well-trained (\Cref{def:well-trained}).
\end{lemma}

\begin{proof}
At stage 1, the curriculum's input has $\x^{(0)}_{\beta, [1]} = \bb_\beta$ 
for all $\beta \in [T]$ (full CoT). Recall that
\begin{align*}
    \Lc^{(1)}(\bte) 
    & = \frac{1}{2B}\sum_{m = n + 1}^{T} 
    \| \Psi_1^\top \mOutput^{(1)} \left(g_m^{(1)} \phi\!\left({\sum_{j = 1}^{n_{h[m]-1}}} \sigma_j(\w^{(1)}_m) \x_j^{(0)}\right) + (1 - g_m^{(1)}) \x_m^{(0)} \right) - \bb_{m} \|^2.
\end{align*}

We split into two cases. \textbf{Case 1:} $m > n_2$. Layer 1 is inactive 
($g_m^{(1)} = 0$); $w^{(1)}_{j, m}$ is structurally decoupled from the loss, 
giving zero gradient. \textbf{Case 2:} $m \in [n + 1, n_2]$. Layer 1 is 
active ($g_m^{(1)} = 1$). Thus, we analyze the gradient signal in detail and 
treat Case 1 first.

\paragraph{Case 1: $m > n_2$.}
Under per-block gating, $g_m^{(1)} = 0$. The gated connection update at 
position $m$ becomes
\[
\x^{(1)}_{m, [2]} = \x^{(0)}_{m, [2]} + (1 - g_m^{(1)}) \x^{(0)}_{m, [1]} 
+ g_m^{(1)} \phi(\hat \z_m^{(1)}) 
= \x^{(0)}_{m, [2]} + \x^{(0)}_{m, [1]},
\]
which does not depend on $\hat \z_m^{(1)}$ and hence not on $\w^{(1)}_m$. Therefore $w^{(1)}_{j, m}$ is absent from $\Lc^{(1)}$, giving 
$\partial \Lc^{(1)} / \partial w_{j, m}^{(1)} = 0$ identically. After 
quantization, $w^{(1)}_{j, m}$ remains zero.

\paragraph{Case 2: $m \in [n + 1, n_2]$.}
For $m \in [n + 1, n_2]$, the loss term reduces to
\[
\frac{1}{2B} \| \Psi_1^\top \mOutput^{(1)} \phi(\hat \z_m^{(1)}) - \bb_m \|^2,
\]
where {$\hat \z_m^{(1)} = \sum_{j = 1}^{n} \sigma_j(\w^{(1)}_m) \x_j^{(0)}$}. 
Standard softmax derivative computations give
\[
\frac{\partial \hat \z_m^{(1)}}{\partial w_{j, m}^{(1)}} 
= \sigma_j(\w_m^{(1)}) (\x_j^{(0)} - \hat \z_m^{(1)}).
\]
The gradient of $\Lc^{(1)}$ w.r.t. $w_{j, m}^{(1)}$ is
\begin{align*}
\frac{\partial \Lc^{(1)}}{\partial w_{j, m}^{(1)}} 
= \frac{\sigma_j(\w_m^{(1)})}{B} 
\langle \Pb_1^\top (\Pb_1 \phi(\hat \z_m^{(1)}) - \bb_m), \phi'(\hat \z_m^{(1)}), 
\x_j^{(0)} - \hat \z_m^{(1)} \rangle.
\end{align*}
At initialization $\W^{(1)} = 0$, {$\sigma_j(\w_m^{(1)}) = 1/n$ 
for $j \in [n]$ and $0$ otherwise.} Expanding $\phi$ and $\phi'$ around 0 via 
$\phi(t) = -1 + ct^2 + O(t^4)$, $\phi'(t) = 2ct + O(t^3)$,
\begin{align}
\frac{\partial \Lc^{(1)}}{\partial w_{j, m}^{(1)}} 
= & {\frac{1}{nB}} \Bigg( -\langle \Pb_1^\top \bb_m, 2c \hat \z_m, 
\x_j - \hat \z_m \rangle \label{eq:cot-grad-signal} \\
& + \langle \Pb_1^\top \Pb_1 (-\onebb_d + c \hat \z_m^2), 2c \hat \z_m, 
\x_j - \hat \z_m \rangle \label{eq:cot-grad-noisy-1} \\
& + \langle O(\Pb_1^\top \Pb_1 |\hat \z_m|^4), 2c \hat \z_m, 
\x_j - \hat \z_m \rangle \label{eq:cot-grad-noisy-2} \\
& + \langle \Pb_1^\top (\Pb_1 \phi(\hat \z_m) - \bb_m), O(|\hat \z_m|^3), 
\x_j - \hat \z_m \rangle \Bigg). \label{eq:cot-grad-noisy-3}
\end{align}
\textbf{Signal term \eqref{eq:cot-grad-signal}.} Substituting 
{$\hat \z_m = \frac{1}{n} \sum_{\alpha \in [n]} \x_\alpha$},
\begin{align*}
& \frac{1}{B} \langle \Pb_1^\top \bb_m, \hat \z_m, \x_j - \hat \z_m \rangle 
\\ & = {\frac{1}{nB}} \left( {\sum_{\alpha \in [n]}} 
\langle \bb_m, \blkx{\alpha}{1}, \blkx{j}{1} \rangle 
- {\frac{1}{n} \sum_{\alpha, \beta \in [n]}} 
\langle \bb_m, \blkx{\alpha}{1}, \blkx{\beta}{1} \rangle \right).
\end{align*}
Since $\Pb_1^\top \bb_m$ is supported on block 1 only and 
$\blkx{\alpha}{1} = \bb_\alpha$ at stage 1's input,
\[
\langle \Pb_1^\top \bb_m, \x_\alpha, \x_\beta \rangle = 
\langle \bb_m, \bb_\alpha, \bb_\beta \rangle.
\]
{The customized mask restricts $\alpha, \beta$ to $[n]$, so 
$h[\alpha] = h[\beta] = 1$. By the parity tree's structure, $\bb_m$ at level 
$h[m] = 2$ covers two disjoint input bits, namely $\bb_{c_1[m]}$ and 
$\bb_{c_2[m]}$. For $\langle \bb_m, \bb_\alpha, \bb_\beta \rangle$ to be trivial, 
$\bb_\alpha \odot \bb_\beta = \bb_m$, which holds iff 
$\{\alpha, \beta\} = \{c_1[m], c_2[m]\}$.}
For non-trivial tuples, \Cref{lemma:trivial-noise} gives 
$|\langle \bb_m, \bb_\alpha, \bb_\beta \rangle| \leq B \kappa$ where 
$\kappa = O(n^{-1 - \epsilon/4})$. Combining,
\[
{\frac{1}{B} \sum_{\alpha, \beta \in [n]}} \langle \bb_m, \bb_\alpha, \bb_\beta \rangle 
= 2 + O(n^2 \kappa).
\]
Similarly, $\langle \bb_m, \bb_\alpha, \bb_j \rangle$ is non-trivial only when 
$p[j] = m$ and $\alpha$ is the other child of $m$. As a result,
\[
{\frac{1}{B} \sum_{\alpha \in [n]}} \langle \bb_m, \bb_\alpha, \bb_j \rangle = 
\begin{cases}
1 + O(n \kappa) & p[j] = m, \\
O(n \kappa) & \text{otherwise}.
\end{cases}
\]
Combining,
\[
-{\frac{1}{nB}} \langle \Pb_1^\top \bb_m, 2c \hat \z_m, \x_j - \hat \z_m \rangle 
= -{\frac{2c}{n^2}} \one{p[j] = m} + O(n^{-2 - \epsilon/4}).
\]
Next for Eq.~\eqref{eq:cot-grad-noisy-1}, we write
\begin{align*}
& \frac{1}{B} \langle \Pb_1^\top \Pb_1(-\onebb_d + c \hat \z_m^2), 
2c \hat \z_m, \x_j - \hat \z_m \rangle   
\\ & = - \frac{2c}{B} \langle \hat \z_{m, [1]}, \x_{j, [1]} \rangle 
+ \frac{2c}{B} \langle \hat \z_{m, [1]}^2 \rangle 
+ \frac{2c^2}{B} \langle \hat \z_{m, [1]}^3, \x_{j, [1]} \rangle 
- \frac{2c^2}{B} \langle \hat \z_{m, [1]}^4 \rangle.
\end{align*}

For the second-order terms, we have:
\begin{align*}
\frac{1}{B}\langle \hat \z_{m, [1]}, \x_{j, [1]} \rangle 
&= {\frac{1}{{n}B}} \left( \langle \bb_j, \bb_j \rangle + 
\sum_{{\alpha \neq j,\, \alpha \in [n]}} \langle \bb_\alpha, \bb_j \rangle \right) 
= {\frac{1}{{n}}} + O(\kappa),
\\ \frac{1}{B}\langle \hat \z_{m, [1]}^2 \rangle 
&= {\frac{1}{{n^2} B}} \left( \sum_{{\alpha \in [n]}} \langle \bb_\alpha, \bb_\alpha \rangle 
+ \sum_{{\alpha \neq \beta,\, \alpha,\beta \in [n]}} \langle \bb_\alpha, \bb_\beta \rangle \right) 
= {\frac{1}{{n}}} + O(\kappa).
\end{align*}

For the fourth-order interaction terms, we discuss when 
$(\alpha, \beta, \gamma, \delta) \notin I_{4, m}$. By \Cref{lem:trivial-4-tuples},  the total count of trivial 4-tuples is $O(n^2)$. Therefore,
\begin{align*}
\frac{1}{B} \langle \hat \z_{m, [1]}^4 \rangle 
&= {\frac{1}{{n^4} B}} \sum_{{\alpha, \beta, \gamma, \delta \in [n]}} 
\langle \bb_\alpha, \bb_\beta, \bb_\gamma, \bb_\delta \rangle
\\ &= {\frac{1}{{n^4} B}} \left( 
\sum_{\alpha, \beta, \gamma, \delta \in I_{4, m}} O(B \kappa) 
+ \sum_{\alpha, \beta, \gamma, \delta \notin I_{4, m}} B \right)
\\ &\leq {\frac{1}{{n^4}}} \left( O(n^2) + {n^4} \kappa \right) 
\\ &= O(n^{-2} + \kappa).
\end{align*}

Now for $\langle \hat \z_{m, [1]}^3, \x_{j, [1]} \rangle$, assuming index 
$j$ is contained in $(\alpha, \beta, \gamma, \delta)$. Then case 1 in \Cref{lem:trivial-4-tuples} has 
$O(1)$ trivial tuples, cases 3 and 4 are reduced to $O(n)$ (one free index). Thus
\[
\frac{1}{B} \langle \hat \z_{m, [1]}^3, \x_{j, [1]} \rangle 
= {\frac{1}{{n^3} B}} \sum_{{\alpha, \beta, \gamma \in [n]}} 
\langle \bb_\alpha, \bb_\beta, \bb_\gamma, \bb_j \rangle 
\leq O(n^{-2} + \kappa).
\]
Combining,
\[
{\frac{1}{{n}B}} \langle \Pb_1^\top \Pb_1(-\onebb_d + c \hat \z_m^2), 
2c \hat \z_m, \x_j - \hat \z_m \rangle 
= {\frac{O(\kappa)}{{n}}} = O(n^{-2 - \epsilon/4}).
\]
For Eq.~\eqref{eq:cot-grad-noisy-2}, note that
\[
\frac{1}{B} \langle \Pb_1^\top \Pb_1 |\hat \z_m|^4 \rangle = 
\frac{1}{B} \langle \hat \z_{m, [1]}^4 \rangle = O(n^{-2} + \kappa),
\]
and since $2c \hat \z_m, \x_j - \hat \z_m$ are contained in $[-1, 1]$ and 
$[-2, 2]$ respectively,
\[
{\frac{1}{{n} B}} \langle O(\Pb_1^\top \Pb_1 |\hat \z_m|^4), 
2c \hat \z_m, \x_j - \hat \z_m \rangle 
= {\frac{4c}{{n}}} O(n^{-2} + \kappa) = O(n^{-2 - \epsilon/4}).
\]
Finally for Eq.~\eqref{eq:cot-grad-noisy-3}, by Cauchy-Schwarz,
\begin{align*}
\frac{1}{B} \langle \Pb_1^\top \Pb_1 |\hat \z_m|^3 \rangle 
= \frac{1}{B} \sum_{i = 1}^B |\hat \z_{m, [1], i}^3|
&\leq \frac{1}{B} \left( \sum_{i = 1}^B \hat \z_{m, [1], i}^2 \right)^{1/2} 
\left( \sum_{i = 1}^B \hat \z_{m, [1], i}^4 \right)^{1/2}
\\ &= O(n^{-1 - \epsilon/8}).
\end{align*}
By the definition of $\Pb_1$,
\begin{align*}
& {\frac{1}{{n} B}} \langle \Pb_1^\top (\Pb_1 \phi(\hat \z_m) - \bb_m), 
O(|\hat \z_m|^3), \x_j - \hat \z_m \rangle 
\\ &= {\frac{1}{{n} B}} \langle \phi(\hat \z_m) - \Pb_1^\top \bb_m, 
O(\Pb_1^\top \Pb_1 |\hat \z_m|^3), \x_j - \hat \z_m \rangle
\\ &= {\frac{4}{{n} B}} \langle O(\Pb_1^\top \Pb_1 |\hat \z_m|^3) \rangle
\\ &= O(n^{-2 - \epsilon/8}).
\end{align*}
Combining all terms from \eqref{eq:cot-grad-signal} to 
\eqref{eq:cot-grad-noisy-3}, {for any $j \in [n]$,} we get
\[
\frac{\partial \Lc^{(1)}}{\partial w_{j, m}^{(1)}} = 
-{\frac{2c}{{n^2}}} \one{p[j] = m} + O(n^{-2 - \epsilon/8}).
\]
Note that this applies to the approximate gradient 
$\tilde \nabla_{w^{(1)}_{j, m}} \Lc^{(1)}$ too since each component of the 
noise is bounded by $O(n^{-2 - \epsilon/8})$.
\paragraph{Concentration of attention scores at $m \in [n + 1, n_2]$.} {The gradient above has a signal of order $n^{-2}$ and a noise remainder 
of order $n^{-2-\epsilon/8}$. For the integer quantization $q(\cdot)$ to map 
both children's updates to the same integer while rounding non-child noise 
to zero, the leading term must land near an integer with margin exceeding 
the post-scaling noise. Under the customized mask in \Cref{fig:mask-loss}(a), 
every query $m \in [n+1, n_2]$ shares the same signal prefactor $n^{-2}$, 
so we use a stage-wise learning rate
\[
\eta_1 := \frac{K_n \, n^2}{2c}, \qquad K_n := \lceil n^{\epsilon/16} \rceil,
\]
where $c > 0$ is the link function 
constant from \Cref{def:link}. This design cancels the prefactor exactly and scales the signal to $K_n$.} 
The update for $m \in [n+1, n_2]$ becomes
\begin{align} \label{eq:updated-weights-base}
w^{(1)}_{j, m}(1) 
= q\!\left( K_n \one{p[j] = m} + \Delta_{j, m} \right),
\end{align}
where $\Delta_{j, m} := \eta_1 R_{j, m}$ and the noise term $R_{j, m}$ 
satisfies $|R_{j, m}| \leq C_\nabla n^{-2 - \epsilon/8}$ uniformly. Since 
$K_n \leq 2 n^{\epsilon/16}$ for sufficiently large $n$, 
\[
|\Delta_{j, m}| \leq \frac{K_n n^2}{2c} \cdot C_\nabla n^{-2 - \epsilon/8} 
\leq C'_\nabla n^{-\epsilon/16} < \tfrac{1}{2}.
\]

\emph{Non-children indices.} For $j \notin \{c_1[m], c_2[m]\}$, the indicator 
$\one{p[j] = m}$ vanishes, so $w^{(1)}_{j, m}(1) = q(\Delta_{j, m}) = 0$.

\emph{Children indices.} For $j \in \{c_1[m], c_2[m]\}$, since $K_n \in \Z$ 
and $|\Delta_{j, m}| < 1/2$,
\[
w^{(1)}_{j, m}(1) = q(K_n + \Delta_{j, m}) = K_n.
\]
Thus $w^{(1)}_{c_1[m], m}(1) = w^{(1)}_{c_2[m], m}(1) = K_n$, while all 
non-child logits are zero.

\emph{Softmax concentration.} The total softmax mass on non-child indices is 
bounded by
\[
\sum_{j \notin \{c_1[m], c_2[m]\}} \sigma_j(\w^{(1)}_m) 
\leq \frac{n - 2}{2 e^{K_n} + (n - 2)} 
\leq n e^{-K_n} 
\leq \exp(-C n^{\epsilon/16})
\]
for some constant $C > 0$ and sufficiently large $n$. Since the two child 
logits are exactly equal,
\[
\frac{1 - \exp(-C n^{\epsilon/16})}{2} 
\;\leq\; \sigma_{c_1[m]}(\w^{(1)}_m) = \sigma_{c_2[m]}(\w^{(1)}_m) 
\;\leq\; \frac{1}{2}.
\]

\paragraph{Forward pass evaluation at $m \in [n + 1, n_2]$.}
For $m \in [n + 1, n_2]$, the children $c_1[m], c_2[m]$ are at level 1 (input bits), 
so $\x^{(0)}_{c_i[m], [1]} = \bb_{c_i[m]}$ exactly. Hence
\begin{align*}
\left\| \hat \z_{m, [1]}^{(1)} - \frac{\bb_{c_1[m]} + \bb_{c_2[m]}}{2} \right\|_\infty 
&\leq \sum_{p[j] \neq m} \sigma_j(\w^{(1)}_m) 
+ \left| \sigma_{c_1[m]}(\w^{(1)}_m) - \tfrac{1}{2} \right| 
+ \left| \sigma_{c_2[m]}(\w^{(1)}_m) - \tfrac{1}{2} \right| 
\\ &\leq {2 \exp(-Cn^{\epsilon/16})}.
\end{align*}
By Taylor expansion of \(\phi\) around
\[
s_m := \frac{b_{c_1[m]}+b_{c_2[m]}}{2}\in\{-1,0,1\},
\]
using
\[
\phi(s_m)=b_{c_1[m]}b_{c_2[m]}=b_m
\qquad\text{and}\qquad
\phi'(s_m)=0,
\]
we obtain
\begin{align}
\| \x^{(1)}_{m, [2]} - \bb_m \|_\infty 
&= \left\| \phi(\hat \z_{m, [1]}^{(1)}) - \phi\!\left( \frac{\bb_{c_1[m]} + \bb_{c_2[m]}}{2} \right) \right\|_\infty 
\notag \\
&\leq C_2 {\bigl( 2 \exp(-Cn^{\epsilon/16}) \bigr)^2} 
\leq \exp(-Cn^{\epsilon/16}). \label{eq:base-forward}
\end{align}
\end{proof}

At each stage $t \geq 1$, we replace the first $k (1 -  2^{- (t - 1)} )$ thinking tokens with padding tokens. This ensures that at $L := \log_2 k $ stages, we are relying only on the input bits to predict. The input sequence looks like
\[
    \bb^{(t)} = (\bb_1, \dots, \bb_n , \underbrace{\bzero_B, \dots , \bzero_B}_{n_t - n}, \bb_{n_t + 1}, \bb_{n_t + 2} , \dots, \bb_{T})
\]
The loss function is only computed on position $n_t + 1 \leq m \leq T$, i.e., 
\[
    \Lc^{(t)}(\bte) = \frac{1}{2B}\sum_{m = n_{t} + 1}^{T}  \| \hat f(\Db)_{m} - \bb_{m} \| ^2
\]
At training time, since the first $k (1 -  2^{- (t - 1)} )$ thinking tokens are masked, the model needs to rely on the hidden thinking traces in $t$-th layer to do the prediction for $n_t + 1 \leq m \leq n_{t + 1}$. Similar to the first stage, we can learn the $t$-th layer by one step of gradient descent as shown in \Cref{lem:inductive-step}. Before proceeding, we first state a lemma that characterizes how the hidden states progress over layers.

\begin{lemma}[Characterization of Hidden States at Stage $s$]\label{lemma:hidden-state}
Fix a curriculum stage $s \in [L]$. Suppose that for all $\tau \in \{1, \ldots, s-1\}$, 
the per-layer recursion satisfies
\begin{equation}\label{eq:icot-induction-res}
\x^{(\tau)}_{m, [\tau + 1]} = 
\begin{cases}
\bb_m + \bxi, & n_\tau + 1 \leq m \leq n_{\tau + 1}, \\
\x^{(\tau - 1)}_{m, [\tau]}, & \text{otherwise},
\end{cases}
\end{equation}
where $\|\bxi\|_\infty \leq \delta := \exp(-Cn^{\epsilon/16})$. Then for any 
layer $\ell \in \{0, 1, \ldots, s-1\}$ and any $\beta \in [T]$,
\begin{equation}\label{eq:icot-induction-state}
\x^{(\ell)}_{\beta, [\ell + 1]} = 
\begin{cases}
\bb_\beta + \bxi, & n + 1 \leq \beta \leq n_{\ell+1}, \\
\bb_\beta, & \beta \leq n \text{ or } \beta > n_s, \\
\bm{0}_B, & n_{\ell+1} < \beta \leq n_s,
\end{cases}
\end{equation}
\end{lemma}

\begin{proof}
We prove Eq.~\eqref{eq:icot-induction-state} by induction on $\ell$.

\paragraph{Base case ($\ell = 0$).}
At $\ell = 0$, the residual stream is the curriculum's stage-$s$ input:
\[
\x^{(0)}_{\beta, [1]} = 
\begin{cases}
\bb_\beta, & \beta \leq n \text{ or } \beta > n_s, \\
\bm{0}_B, & n + 1 \leq \beta \leq n_s .
\end{cases}
\]
Since $n_{\ell+1} = n_1 = n$, the ``computed'' range $n+1 \leq \beta \leq n_{\ell+1}$ 
is empty, and the three cases above coincide with the three cases of 
Eq.~\eqref{eq:icot-induction-state} at $\ell = 0$.

\paragraph{Inductive step ($\ell \Rightarrow \ell + 1$, for $\ell + 1 \leq s - 1$).}
Assume Eq.~\eqref{eq:icot-induction-state} holds at layer $\ell$. We verify 
it at layer $\ell + 1$ by partitioning $[T]$ into two regions and applying 
Eq.~\eqref{eq:icot-induction-res} at $\tau = \ell + 1$:

\textbf{Region 1: Newly computed positions ($n_{\ell+1} + 1 \leq \beta \leq n_{\ell+2}$).}
By Eq.~\eqref{eq:icot-induction-res},
\[
\x^{(\ell+1)}_{\beta, [\ell+2]} = \bb_\beta + \bxi, \quad \|\bxi\|_\infty \leq \delta.
\]
Since $n + 1 \leq \beta \leq n_{\ell+2}$, this matches the computed-range 
case of Eq.~\eqref{eq:icot-induction-state} at $\ell + 1$.

\textbf{Region 2: Inherited positions ($\beta \leq n_{\ell+1}$ or $\beta > n_{\ell+2}$).}
For these positions, Eq.~\eqref{eq:icot-induction-res} gives 
$\x^{(\ell+1)}_{\beta, [\ell+2]} = \x^{(\ell)}_{\beta, [\ell+1]}$, so the 
value is inherited from layer $\ell$ and we apply the inductive hypothesis:
\begin{itemize}
\item If $n + 1 \leq \beta \leq n_{\ell+1}$ (previously computed), then 
$\x^{(\ell)}_{\beta, [\ell+1]} = \bb_\beta + \bxi$ with 
$\|\bxi\|_\infty \leq \delta$. Since $\beta \in [n+1, n_{\ell+2}]$, 
this matches the computed-range case at $\ell + 1$.
\item If $\beta \leq n$ or $\beta > n_s$ (shown), then 
$\x^{(\ell)}_{\beta, [\ell+1]} = \bb_\beta$, matching the shown-range case 
at $\ell + 1$.
\item If $n_{\ell+2} < \beta \leq n_s$ (padded), then $\beta$ also satisfies 
$n_{\ell+1} < \beta \leq n_s$, so by the inductive hypothesis 
$\x^{(\ell)}_{\beta, [\ell+1]} = \bm{0}_B$, matching the padded-range case 
at $\ell + 1$.
\end{itemize}

In both regions, Eq.~\eqref{eq:icot-induction-state} holds at layer 
$\ell + 1$. This completes the induction.
\end{proof}

\begin{lemma}[Inductive Step]
\label{lem:inductive-step}
Fix any $t \in \{1, \dots, L-1\}$. 
Suppose layers $1, \dots, t$ are well-trained 
(\Cref{def:well-trained}) and $\W^{(\ell)}(t) = 0$ for all 
$\ell \geq t+1$. After one gradient step on the stage-$(t+1)$ loss 
$\mathcal{L}^{(t+1)}$ with $B$ training samples and stage-wise learning rate 
$\eta_{t+1} = K_n n_{t+1}^2 / (2c)$, where $B = \Omega(n^{2+\epsilon})$, $K_n := \lceil n^{\epsilon/16} \rceil$ and \(c>0\) is the link function constant in \Cref{def:link}, with probability $1 - \exp(-n^{\epsilon/2})$ 
over the fresh stage-$(t+1)$ batch, layer $t+1$ is also well-trained 
(\Cref{def:well-trained}).
\end{lemma}

\begin{proof}
By the assumption that layers $1, \dots, t$ are well-trained and 
\Cref{lemma:hidden-state} (applied with $s = t+1$, $\ell = t$), 
the hidden states satisfy 
$\x^{(t)}_{\tau, [t+1]} = \bb_\tau + \bxi$ with 
$\|\bxi\|_\infty \leq \delta = \exp(-C n^{\epsilon/16})$ for all $\tau \in [T]$. Recall that
\begin{align*}
& \Lc^{(t + 1)}(\bte) 
\\ & = \frac{1}{2B}\sum_{m = n_{t + 1} + 1}^{T} 
\| \Psi_{t + 1}^\top \mOutput^{(t + 1)} \left(g_m^{(t + 1)} \phi\!\left({\sum_{j = 1}^{n_{h[m]-1}}} \sigma_j(\w^{(t + 1)}_m) \x_j^{(t)}\right) + (1 - g_m^{(t + 1)}) \x_m^{(t)} \right) - \bb_{m} \|^2.
\end{align*}

We split into two cases. 

\textbf{Case 1: $m \in [n_{t+1} + 1, n_{t+2}]$.} 
Here $g_m^{(t+1)} = 1$ (active) and the gradient encodes a non-trivial signal 
from the contraction.

\textbf{Case 2: $m > n_{t+2}$.} Here $g_m^{(t+1)} = 0$ 
by per-block gating, layer $t + 1$ does not write at position $m$, and the gradient is identically zero. We treat case 1 in detail and then briefly address case 2 at the end of the proof.

\medskip
\textbf{Case 1: $m \in [n_{t+1} + 1, n_{t+2}]$.} 
Under the causal mask, {$\sigma_j(\w_m) = 1/n_{t+1}$ for $j \in [n_{t+1}]$ and 
$0$ otherwise} at initialization $\W^{(t+1)} = 0$.
\begin{align}
\frac{\partial \Lc^{(t + 1)}}{\partial w_{j, m}^{(t+1)}}
= & \frac{\sigma_j(\w_m^{(t + 1)})}{B} 
\langle 
\Pb_{t + 1}^\top(\Pb_{t + 1}\phi(\hat \z_m^{(t + 1)}) - \bb_{m}),
\phi'(\hat \z_m^{(t + 1)}),
\x_j^{(t)} - \hat \z_m^{(t + 1)}
\rangle \notag 
\\ = & -{\frac{1}{n_{t+1} B}} \langle 
\Pb_{t + 1}^\top \bb_{m},
2c \hat \z_m^{(t + 1)},
\x_j^{(t)} - \hat \z_m^{(t + 1)}
\rangle 
\label{eq:icot-grad-signal}
\\ & + {\frac{1}{n_{t+1} B}} \langle 
\Pb_{t + 1}^\top \Pb_{t + 1}(-\onebb_d + c(\hat \z_m^{(t + 1)})^2),
2c \hat \z_m^{(t + 1)},
\x_j^{(t)} - \hat \z_m^{(t + 1)}
\rangle 
\label{eq:icot-grad-noisy-1}
\\ & + {\frac{1}{n_{t+1} B}} \langle 
O(\Pb_{t + 1}^\top \Pb_{t + 1} |\hat \z_m^{(t + 1)}|^4),
2c \hat \z_m^{(t + 1)},
\x_j^{(t)} - \hat \z_m^{(t + 1)}
\rangle 
\label{eq:icot-grad-noisy-2}
\\ & + {\frac{1}{n_{t+1} B}} \langle 
\Pb_{t + 1}^\top(\Pb_{t + 1}\phi(\hat \z_m^{(t + 1)}) - \bb_{m}),
O(|\hat \z_m^{(t + 1)}|^3),
\x_j^{(t)} - \hat \z_m^{(t + 1)}
\rangle .
\label{eq:icot-grad-noisy-3}
\end{align}

For \Cref{eq:icot-grad-signal}, we substitute 
{$\hat \z_m^{(t + 1)} = \frac{1}{n_{t+1}} \sum_{\alpha \in [n_{t+1}]} \x^{(t)}_\alpha$}
to expand
\begin{align*}
& \frac{1}{B} \langle \Pb_{t + 1}^\top \bb_{m}, \hat \z_m^{(t + 1)}, \x_j^{(t)} - \hat \z_m^{(t + 1)} \rangle 
\\ &= {\frac{1}{n_{t+1} B}} \left( {\sum_{\alpha \in [n_{t+1}]}} \langle \bb_m,  \blkx{\alpha}{t + 1} ^{(t)} , \blkx{j}{t+1}^{(t)} \rangle - {\frac{1}{n_{t+1}} \sum_{\alpha, \beta \in [n_{t+1}]}} \langle \bb_m,  \blkx{\alpha}{t + 1} ^{(t)} , \blkx{\beta}{t + 1} ^{(t)} \rangle \right).
\end{align*}
{The customized mask restricts $\alpha, \beta$ to $[n_{t+1}]$, so $h[\alpha], h[\beta] \leq t+1 < h[m] = t+2$. By the parity tree's structure, distinct same-level nodes cover disjoint subsets of input bits, and $\bb_m$ covers $2^{t+1}$ disjoint input bits at level $t+2$. For $\langle \bb_m, \bb_\alpha, \bb_\beta \rangle$ to be trivial, the bit multiplicities must align, which forces $\bb_\alpha \odot \bb_\beta = \bb_m$, i.e., $\{\alpha, \beta\} = \{c_1[m], c_2[m]\}$ at level $t+1$.} Hence,
\begin{align*}
    \left| \langle \bb_{m},  \blkx{\alpha}{t + 1}^{(t)}, \blkx{\beta}{t + 1}^{(t)} \rangle \right| 
    &\leq |\langle \bb_{m}, \bb_\alpha, \bb_\beta \rangle| + 2\delta B + \delta^2 B 
    = B(1 + O(\delta)),
\end{align*}
and for non-trivial $(\alpha, \beta)$,
\begin{align*}
    \left| \langle \bb_{m},  \blkx{\alpha}{t + 1}^{(t)}, \blkx{\beta}{t + 1}^{(t)} \rangle \right| 
    \leq B\kappa + 2\delta B + \delta^2 B = O(B(\kappa + \delta)).
\end{align*}
{Using $n_{t+1} = \Theta(n)$ (specifically $n \leq n_{t+1} < 2n$),} putting these together,
\begin{align*}
    {\frac{1}{B} \sum_{\alpha, \beta \in [n_{t+1}]}} \langle \bb_{m},  \blkx{\alpha}{t + 1} ^{(t)} , \blkx{\beta}{t + 1} ^{(t)} \rangle 
    = 2 + O(n^2 \delta) + O(n^2 \kappa) = 2 + O(n^2 \kappa).
\end{align*}
Similarly, $\langle \bb_{m}, \blkx{\alpha}{t+1}^{(t)}, \blkx{j}{t+1}^{(t)} \rangle$ is non-trivial only when $p[j] = m$ and $\alpha$ is the other child. By the same expansion,
\[
{\frac{1}{B} \sum_{\alpha \in [n_{t+1}]}} \langle \bb_{m}, \blkx{\alpha}{t+1}^{(t)}, \blkx{j}{t+1}^{(t)} \rangle = \begin{cases}
    1 + O(n\kappa) & p[j] = m, \\
    O(n\kappa) & \text{o.w.}
\end{cases}
\]
Combining these with $\kappa = O(n^{-1-\epsilon/4})$,
\begin{align*}
-{\frac{1}{n_{t+1} B}} \langle \Pb_{t+1}^\top \bb_m, 2c\hat\z_m^{(t+1)}, \x_j^{(t)} - \hat\z_m^{(t+1)} \rangle 
&= -{\frac{2c}{n_{t+1}^2}} \one{p[j]=m} + O(n^{-2-\epsilon/4}).
\end{align*}
Next for \Cref{eq:icot-grad-noisy-1}, we expand
\begin{align*}
& \frac{1}{B} \langle \Pb_{t + 1}^\top \Pb_{t + 1}(-\onebb_d + c (\hat \z_m^{(t + 1)})^2) ,  2c \hat \z_m^{(t + 1)} , \x_j ^{(t)} - \hat \z_m^{(t + 1)} \rangle   
\\ =& - \frac{2c}{B} \langle \hat \z_{m, [t + 1]}^{(t + 1)}, \x_{j, [t + 1]}^{(t)} \rangle + \frac{2c}{B} \langle (\hat \z_{m, [t + 1]}^{(t + 1)})^2 \rangle + \frac{2c^2}{B} \langle (\hat \z_{m, [t + 1]}^{(t + 1)})^3, \x_{j, [t + 1]}^{(t)} \rangle - \frac{2c^2}{B} \langle (\hat \z_{m, [t + 1]}^{(t + 1)})^4 \rangle.
\end{align*}
For the second-order terms, we have:
\begin{align*}
\frac{1}{B}\langle \hat \z_{m, [t + 1]}^{(t + 1)}, \x_{j, [t + 1]}^{(t)} \rangle 
&= {\frac{1}{{n_{t + 1}}B}} \left( \langle \x_{j, [t + 1]}^{(t)}, \x_{j, [t + 1]}^{(t)} \rangle  + {\sum_{\alpha \neq j, \alpha \in {[n_{t+1}]}}} \langle \x_{\alpha, [t + 1]}^{(t)}, \x_{j, [t + 1]}^{(t)} \rangle \right)
\\ &= {\frac{1}{{n_{t + 1}}B}} \left( \langle \bb_j, \bb_j \rangle + \sum_{\alpha \neq j} \langle \bb_\alpha, \bb_j \rangle \right) + {\frac{O(n)}{{n_{t+1}}}} (\delta  + \delta^2)
\\ &= {\frac{1}{{n_{t+1}}}} (1 + {(n_{t+1} - 1)} \kappa) + {\frac{O(n)}{{n_{t+1}}}} (\delta + \delta^2)
\\ &= {\frac{1}{{n_{t+1}}}} + O(\kappa).
\end{align*}

Similarly,
\begin{align*}
\frac{1}{B}\langle (\hat \z_{m, [t + 1]}^{(t + 1)})^2 \rangle 
&= {\frac{1}{{n_{t + 1}^2} B}} \left( \sum_\alpha \langle \x_{\alpha, [t + 1]}^{(t)}, \x_{\alpha, [t + 1]}^{(t)} \rangle  + \sum_{\alpha \neq \beta} \langle \x_{\alpha, [t + 1]}^{(t)}, \x_{\beta, [t + 1]}^{(t)} \rangle \right) 
\\ &= {\frac{1}{{n_{t+1}^2} B}} \left( \sum_\alpha \langle \bb_\alpha, \bb_\alpha \rangle + \sum_{\alpha \neq \beta} \langle \bb_\alpha, \bb_\beta \rangle \right) + {\frac{O(n)}{{n_{t+1}}}} \delta + {\frac{O(n^2)}{{n_{t+1}^2}}} \delta^2 
\\ &= {\frac{1}{{n_{t+1}}}} + O(\kappa).
\end{align*}

For the fourth-order interaction terms, we discuss when 
$(\alpha, \beta, \gamma, \delta) \notin I_{4,m}$, i.e., $\langle \bb_\alpha, \bb_\beta, \bb_\gamma, \bb_\delta \rangle$ is trivial. Again, by the bit-counting argument in \Cref{lem:trivial-4-tuples}, the number of trivial ones are at most $O(n^2)$. Hence,
\begin{align*}
\frac{1}{B} \langle (\hat \z_{m, [t + 1]}^{(t + 1)})^4 \rangle 
& = {\frac{1}{{n_{t+1}^4} B}} \sum_{{\alpha,\beta,\gamma,\delta \in [n_{t+1}]}} \langle \x_{\alpha, [t + 1]}^{(t)}, \x_{\beta, [t + 1]}^{(t)}, \x_{\gamma, [t + 1]}^{(t)}, \x_{\delta, [t + 1]}^{(t)}\rangle
\\ & = {\frac{1}{{n_{t+1}^4} B}} \sum_{{\alpha,\beta,\gamma,\delta \in [n_{t+1}]}} \langle \bb_\alpha, \bb_\beta, \bb_\gamma, \bb_\delta \rangle + \sum_{j = 1}^3 {\frac{O(n^j)}{{n_{t+1}^j}}} \delta^j \kappa + {\frac{O(n^4)}{{n_{t+1}^4}}} \delta^4
\\ & \leq {\frac{1}{{n_{t+1}^4} B}} \left( \sum_{\alpha, \beta, \gamma,\delta \in I_{4, m}} O(B \kappa) + \sum_{\alpha, \beta, \gamma,\delta \notin I_{4, m}} B \right) + O(\delta)
\\ & \leq {\frac{1}{{n_{t+1}^4}}} (O(n^2) + {n_{t+1}^4} \kappa) + O(\delta)
\\ & = O(n^{-2} + \kappa).
\end{align*}

Now for $\langle (\hat \z_{m, [t + 1]}^{(t + 1)})^3, \x_{j, [t + 1]}^{(t)} \rangle$, 
assuming index $j$ is contained in $(\alpha, \beta, \gamma, \delta)$. Repeating the argument in \Cref{lemma:base}, we get
\begin{align*}
\frac{1}{B} \langle (\hat \z_{m, [t + 1]}^{(t + 1)})^3, \x_{j, [t + 1]}^{(t)} \rangle 
&\leq {\frac{1}{{n_{t+1}^3} B}} \sum_{{\alpha,\beta,\gamma \in [n_{t+1}]}} \langle \bb_\alpha, \bb_\beta, \bb_\gamma, \bb_j \rangle + O(n^{-1 - \epsilon / 4})
\\ & \leq {\frac{1}{{n_{t+1}^3}}} (O(n) + {n_{t+1}^3} \kappa) + O(\delta)
\\ & = O(n^{-2} + \kappa).
\end{align*}

For Eq.~\eqref{eq:icot-grad-noisy-2}, note that
\begin{align*}
\frac{1}{B}\langle \Pb_{t + 1}^\top \Pb_{t + 1} |\hat \z_m^{(t + 1)}|^4 \rangle = \frac{1}{B} \langle |\hat \z_{m, [t + 1]}^{(t + 1)}|^4 \rangle = \frac{1}{B} \langle (\hat \z_{m, [t + 1]}^{(t + 1)})^4 \rangle = O(n^{-2} + \kappa).
\end{align*}
and since $2c \hat \z_m^{(t + 1)} ,\x_j ^{(t)} - \hat \z_m^{(t + 1)}$ are contained in $[-1, 1]$ and $[-2,2]$ respectively. We have
\begin{align*}
{\frac{1}{{n_{t+1}}B}} \langle \Pb_{t + 1}^\top \Pb_{t + 1} |\hat \z_m^{(t + 1)}|^4 ,  2c \hat \z_m^{(t + 1)} , \x_j ^{(t)} - \hat \z_m^{(t + 1)} \rangle  
& = {\frac{4c \cdot O(n^{-2} + \kappa)}{{n_{t+1}}}} 
= O(n^{-2 - \epsilon / 4}).
\end{align*}

Finally for Eq.~\eqref{eq:icot-grad-noisy-3}, using Cauchy-Schwarz we have
\begin{align*}
\frac{1}{B} \langle  \Pb_{t + 1}^\top \Pb_{t + 1} |\hat \z_m^{(t + 1)}|^3 \rangle 
= \frac{1}{B} \langle |( \hat \z_{m, [t + 1]}^{(t + 1)})^3| \rangle 
& = \frac{1}{B} \sum_{i = 1}^B | (\hat \z_{m, [t + 1], i}^{(t + 1)})^3|  
\\ & \leq \frac{1}{B} \left (\sum_{i = 1}^B  (\hat \z_{m, [t + 1], i}^{(t + 1)})^2 \right)^{1/2} \left (\sum_{i = 1}^B  (\hat \z_{m, [t + 1], i}^{(t + 1)})^4 \right)^{1/2} 
\\ &= \frac{1}{B} O(B n^{-1})^{1/2} O(B n^{-1 - \epsilon / 4}  )^{1/2} 
\\ & = O(n^{-1 - \epsilon / 8}).
\end{align*}
By the definition of $\Pb_{t+1}$, we get
\begin{align*}
& {\frac{1}{{n_{t+1}} B}} \langle 
\Pb_{t + 1}^\top(\Pb_{t + 1}\phi(\hat \z_m^{(t + 1)}) - \bb_{m}),
O(|\hat \z_m^{(t + 1)}|^3),
\x_j^{(t)} - \hat \z_m^{(t + 1)}
\rangle 
\\ &=  {\frac{1}{{n_{t+1}} B}}     \langle 
\phi(\hat \z_m^{(t + 1)}) - \Pb_{t + 1}^\top \bb_{m},
O(\Pb_{t + 1}^\top\Pb_{t + 1}|\hat \z_m^{(t + 1)}|^3),
\x_j^{(t)} - \hat \z_m^{(t + 1)}
\rangle 
\\ & =  {\frac{4}{{n_{t+1}} B}}    O \left(\langle \Pb_{t + 1}^\top \Pb_{t + 1} |\hat \z_m^{(t + 1)}|^3 \rangle \right)
\\ & = O(n^{-2 - \epsilon / 8}).
\end{align*}
Combining all terms from \eqref{eq:icot-grad-signal} to \eqref{eq:icot-grad-noisy-3}, for any {$j \in [n_{t+1}]$}, we get
\[
     \frac{\partial \Lc^{(t + 1)}}{\partial w_{j, m}^{(t + 1)}}  = - {\frac{2c}{{n_{t+1}^2}}} \one{p[j] = m} + O(n^{-2- \epsilon / 8})
\]

{\textbf{Case 2: $m > n_{t + 2}$.} Under per-block gating, $g_m^{(t+1)} = 0$. 
The gated connection update at position $m$ becomes
\[
\x^{(t+1)}_{m, [t+2]} = \x^{(t)}_{m, [t+2]} + (1 - g_m^{(t+1)}) \x^{(t)}_{m, [t+1]} + g_m^{(t+1)} \phi(\hat \z_m^{(t+1)}) = \x^{(t)}_{m, [t+2]} + \x^{(t)}_{m, [t+1]},
\]
which does not depend on $\hat \z_m^{(t+1)}$ and hence not on $\w^{(t+1)}_m$. Therefore, $w^{(t+1)}_{j, m}$ is absent from the loss $\Lc^{(t+1)}$ entirely, 
giving
\[
\frac{\partial \Lc^{(t+1)}}{\partial w_{j, m}^{(t+1)}} = 0
\]
identically. After quantization, $w^{(t+1)}_{j, m}$ remains zero.}

\paragraph{Concentration of attention scores at \(m \in [n_{t+1}+1,n_{t+2}]\).} Similar to \Cref{lemma:base}, we choose a single stage-wise learning rate to ensure the non-child noise is rounded to zero using quantization,
\[
\eta_{t+1} := \frac{K_n \, n_{t+1}^2}{2c}.
\]
where $K_n = \lceil n^{\epsilon/16} \rceil$ and where $c > 0$ is the link function 
constant from \Cref{def:link}.
Equivalently, the update for
\(m\in[n_{t+1}+1,n_{t+2}]\) is
\begin{align}
\label{eq:updated-weights-inductive-preconditioned}
    w^{(t+1)}_{j,m}(t+1)
    &=
    q\!\left(
    -\eta_{t + 1}
    \widetilde\nabla_{w^{(t+1)}_{j,m}}\Lc^{(t+1)}
    \right).
\end{align}
Using the gradient expansion,
\[
-\widetilde\nabla_{w^{(t+1)}_{j,m}}\Lc^{(t+1)}
=
\frac{2c}{n_{t + 1}^2}\one{p[j]=m}
+
R_{j,m},
\]
where the noise term satisfies
\[
|R_{j,m}|\le C_\nabla n^{-2-\epsilon/8}
\]
uniformly over \(j\in[n_{t + 1}]\) and
\(m\in[n_{t+1}+1,n_{t+2}]\), we obtain
\begin{align}
    w^{(t+1)}_{j,m}(t+1)
    &=
    q\!\left(
    K_n\one{p[j]=m}
    +
    \Delta_{j,m}
    \right),
\end{align}
where
\[
\Delta_{j,m}
:=
\eta_{t + 1} R_{j,m}.
\]
Since \(n_{t + 1}\le 2n\) on the trainable range and
\(K_n\le 2n^{\epsilon/16}\) for all sufficiently large \(n\), we have
\[
|\Delta_{j,m}|
\le
\frac{K_nn_{t + 1}^2}{2c} C_\nabla n^{-2-\epsilon/8}
\le
C'_\nabla n^{-\epsilon/16},
\]
for a constant \(C'_\nabla>0\). Hence
\[
|\Delta_{j,m}|<\frac12
\]
for all sufficiently large \(n\).

\emph{Non-child indices.}
For \(j\notin\{c_1[m],c_2[m]\}\), the indicator \(\one{p[j]=m}\) vanishes.
Therefore
\[
w^{(t+1)}_{j,m}(t+1)
=
q(\Delta_{j,m})
=
0.
\]

\emph{Children indices.}
For \(j\in\{c_1[m],c_2[m]\}\), we have \(\one{p[j]=m}=1\). Since
\(K_n\in\mathbb Z\) and \(|\Delta_{j,m}|<1/2\),
\[
w^{(t+1)}_{j,m}(t+1)
=
q(K_n+\Delta_{j,m})
=
K_n.
\]
Thus
\[
w^{(t+1)}_{c_1[m],m}(t+1)
=
w^{(t+1)}_{c_2[m],m}(t+1)
=
K_n,
\]
while all non-child logits are zero.

\emph{Softmax concentration.}
The total softmax mass on non-child indices is bounded by
\[
\sum_{j\notin\{c_1[m],c_2[m]\}}
\sigma_j(\w^{(t+1)}_m)
\le
\frac{n}{2\exp(K_n)}
\le
n\exp(-K_n)
\le
\exp\!\left(-C n^{\epsilon/16}\right)
\]
for some constant \(C>0\) and all sufficiently large \(n\). Since the two
child logits are exactly equal, the two children receive the same softmax
weight. Therefore
\[
\frac{1-\exp(-C n^{\epsilon/16})}{2}
\le
\sigma_{c_1[m]}(\w^{(t+1)}_m)
=
\sigma_{c_2[m]}(\w^{(t+1)}_m)
\le
\frac12 .
\]

\paragraph{Evaluating the forward pass.}
To bound the prediction loss, we define the cumulative approximation error after stage $t$ as
\[
\epsilon_t := \max_{1 \leq s \leq t} \, \max_{m \in {[n_s + 1, n_{s+1}]}} {\|\bm{x}^{(s)}_{m,[s+1]} - \bm{b}_{m}\|_\infty}, \quad \epsilon_0 := 0.
\]
By construction, $\epsilon_t$ is non-decreasing in $t$ and bounds the error at every CoT position produced by layers $1, \ldots, t$. In particular, for any {$m \in [n_{t+1} + 1, n_{t+2}]$}, the two children {$c_1[m], c_2[m]$} both lie at levels $\leq t + 1$, so they were produced by some layer $s \leq t$ and satisfy
\[
{\|\bm{x}^{(t)}_{c_1[m],[t+1]} - \bm{b}_{c_1[m]}\|_\infty, \;\; \|\bm{x}^{(t)}_{c_2[m],[t+1]} - \bm{b}_{c_2[m]}\|_\infty \leq \epsilon_t.}
\]
    and for the intermediate state $\hat \z_{m, [t + 1]}^{(t+1)}$ we have
    \begin{align}
        & \left \|\hat \z_{m, [t + 1]}^{(t+1)} - {\frac{\bb_{c_1[m]} + \bb_{c_2[m]}}{2}}  \right \|_\infty  \notag
         \\ &\leq     \left \|\hat \z_{m, [t + 1]}^{(t+1)} - {\frac{ \x_{c_1[m], [t + 1]}^{(t)} + \x_{c_2[m], [t + 1]}^{(t)}}{2}}  \right \|_\infty + \epsilon_{t} \notag
        \\ & \leq   {\sum_{p[j] \neq m} \sigma_j (\w^{(t+1)}_m) + \left |\sigma_{c_1[m]} (\w^{(t+1)}_m) - \frac{1}{2} \right | + \left |\sigma_{c_2[m]} (\w^{(t+1)}_m) - \frac{1}{2} \right |} + \epsilon_{t} \notag
        \\ & \leq 2 \exp (-C n^{\epsilon / 16}) + \epsilon_t .\label{eq:upper-intermediate}
    \end{align}
Therefore, using the Taylor expansion of $\phi$, every newly produced position at stage $t+1$ satisfies
\begin{align*}
    \| \x_{m, [t + 2]}^{(t+1)} - \bb_{m} \|_\infty
    &= \left\| \phi(\hat \z_{m, [t + 1]}^{(t+1)}) - \phi \!\left(\frac{\bb_{c_1[m]} + \bb_{c_2[m]}}{2} \right) \right\|_\infty 
    \\ & \leq C_2 \bigl(2 \exp(-C n^{\epsilon / 16}) + \epsilon_t\bigr)^2,
\end{align*}
for some constant $C_2$ depending on $\phi$ only. 
Combined with the cumulative max, this gives the recursion 
\[
\epsilon_{t+1} \leq \max\!\Bigl\{\epsilon_t, \; C_2 \bigl(2 \exp(-Cn^{\epsilon/16}) + \epsilon_t\bigr)^2 \Bigr\}.
\]
We inductively show that $\epsilon_t \leq \delta := \exp(-Cn^{\epsilon/16})$ for all $t$. The base case $\epsilon_1 \leq \delta$ holds by \Cref{lemma:base}. Assume $\epsilon_t \leq \delta$. The first argument of the $\max$ is at most $\delta$ by hypothesis, and the second is bounded by
\[
C_2 (2\delta + \delta)^2 = 9 C_2 \delta^2.
\]
To conclude $\epsilon_{t+1} \leq \delta$, it suffices that $9 C_2 \delta \leq 1$, which holds for all sufficiently large $n$ since $\delta = \exp(-Cn^{\epsilon/16})$ decays faster than any polynomial. Thus $\epsilon_t \leq \delta$ for all $t \leq L = \log_2 k$ by induction. In particular, since position $m = T$ at the final stage $L$ is included in the cumulative max, we conclude
\[
\|\bm{x}^{(L)}_{T,[L+1]} - \bm{b}_T\|_\infty = \|\hat{f}(\bm{D}_{\text{test}}) - \y_{\text{test}}\|_\infty \leq \exp(-Cn^{\epsilon/16}).
\]
\end{proof}

We are now ready to combine the lemmas of this section into a proof 
of \Cref{thm:icot-cvg}.

\begin{proof}[Proof of \Cref{thm:icot-cvg}]
We prove by induction on $t \in \{1, \dots, L\}$ that after stage $t$, 
layers $1, \dots, t$ are well-trained (\Cref{def:well-trained}) and 
$\W^{(\ell)}(t) = 0$ for all $\ell \geq t+1$.

\paragraph{Base case ($t = 1$).} 
At initialization, $\W^{(\ell)}(0) = 0$ for all $\ell \in [L]$ 
(Algorithm 1). The stage-1 update modifies only $\W^{(1)}$, leaving 
$\W^{(\ell)}(1) = 0$ for all $\ell \geq 2$. By \Cref{lemma:base}, with 
probability $1 - \exp(-n^{\epsilon/2})$ over the stage-1 batch, layer 1 
is well-trained.

\paragraph{Inductive step ($t \to t+1$).} 
Assume that after stage $t$, layers $1, \dots, t$ are well-trained and 
$\W^{(\ell)}(t) = 0$ for all $\ell \geq t+1$. We show the same holds at 
stage $t+1$.

The inductive hypothesis verifies the assumption of 
\Cref{lemma:grad-upper-bounds}. Applying it, for every $\ell' \leq t$, 
$m \in [n_{\ell'} + 1, n_{\ell'+1}]$, and $j \in [n_{\ell'}]$,
\[
\left| \frac{\partial \mathcal{L}^{(t+1)}}{\partial w^{(\ell')}_{j,m}} \right| 
\leq O\big(\exp(-C n^{\epsilon/16})\big),
\]
and the quantized stage-$(t+1)$ update leaves 
$w^{(\ell')}_{j,m}(t+1) = w^{(\ell')}_{j,m}(t)$. Hence layers 
$1, \dots, t$ remain well-trained after stage $t+1$. The inductive 
hypothesis also verifies the assumption of \Cref{lem:inductive-step}. 
Applying it, with probability $1 - \exp(-n^{\epsilon/2})$ over the 
fresh stage-$(t+1)$ batch, layer $t+1$ is well-trained. The stage-$(t+1)$ 
update modifies only $\W^{(t+1)}$, so $\W^{(\ell)}(t+1) = \W^{(\ell)}(t) = 0$ 
for all $\ell \geq t+2$. This completes the inductive step.

\paragraph{Union bound and conclusion.} 
Each stage's well-trainedness conclusion holds with probability 
$1 - \exp(-n^{\epsilon/2})$ over its fresh batch, and the $L = \log_2 k$ 
batches are independent. By a union bound, with probability at least
\[
1 - L \exp(-n^{\epsilon/2}) = 1 - \exp(-\Omega(n^{\epsilon/2}))
\]
(using $L = O(\log n)$), every layer $\ell \in [L]$ is well-trained 
after stage $L$. By \Cref{def:well-trained} applied at layer $L$ 
and position $m = T$ (which lies in $[n_L + 1, n_{L+1}]$ since 
$n_{L+1} = T$),
\[
\|\x^{(L)}_{T, [L+1]} - \bb_T\|_\infty \leq \exp(-C n^{\epsilon/16}).
\]
By \Cref{lemma:hc-propagation} applied at $\beta = T$, $\ell = L$, 
the test-time output at position $T$ satisfies
\[
f^L_{\hat\bte}(\Db_{\test})_T = \Psi_L^\top \Tc_{\hat\bte}(\Db_{\test})_T 
= \x^{(L)}_{T, [L+1]},
\]
and the target is $\y_{\text{test}} = \bb_T$. Therefore
\[
\| f_{\hat \bte}^{L} (\Db_{\test})_T - \y_{\test} \|_\infty 
\leq \exp(-C n^{\epsilon/16}) = \exp(-\Omega(n^{\epsilon/16})),
\]
which establishes the theorem.
\end{proof}

\subsection{Gradient Upper Bound for Trained Layers} \label{app:upper-bound}

\begin{lemma}[Frozen gradient]\label{lemma:grad-upper-bounds}
Fix any $t \in \{1, \dots, L-1\}$. Suppose layers 
$1, \dots, t$ are well-trained (\Cref{def:well-trained}) and 
$\W^{(\ell)}(t) = 0$ for all $\ell \geq t+1$. Then during one stage-$(t+1)$ 
gradient step, for any $\ell' \leq t$, any $m \in [n_{\ell'} + 1, n_{\ell'+1}]$, 
and any {$j \in [n_{\ell'}]$},
\[
\left| \frac{\partial \mathcal{L}^{(t+1)}}{\partial w^{(\ell')}_{j,m}} \right| 
\leq O\big(\exp(-Cn^{\epsilon/16})\big).
\]
Consequently, with {the stage-wise learning rate $\eta_{t+1} = \Theta(K_n n_{t+1}^2)$ 
(where $K_n := \lceil n^{\epsilon/16} \rceil$)}, the weights in all previously 
trained layers remain unchanged after quantization:
\[
w^{(\ell')}_{j,m}(t+1) = w^{(\ell')}_{j,m}(t) \quad 
\text{for all } \ell' \leq t,\; m \in [n_{\ell'}+1, n_{\ell'+1}],\; {j \in [n_{\ell'}]}.
\]
\end{lemma}
\begin{proof}
Throughout, fix $\ell' \leq t$ and a parameter $w^{(\ell')}_{j, m}$ with 
$n_{\ell'} + 1 \leq m \leq n_{\ell' + 1}$ and {$1 \leq j \leq n_{\ell'}$}. 
Let $L_\phi := \|\phi'\|_\infty$.

\paragraph{Step 1: Vanishing of $\phi'$ at well-trained positions.} Recall from \Cref{lem:inductive-step} that, by the assumption on the softmax weights, for any $\ell'' \leq t$ and 
$n_{\ell''} + 1 \leq \alpha \leq n_{\ell''+1}$,
\begin{align}
    \left\| \hat \z^{(\ell'')}_{\alpha, [\ell'']} - 
    \frac{\bb_{c_1[\alpha]} + \bb_{c_2[\alpha]}}{2} \right\|_\infty 
    \leq (2n - 1) \exp(-C n^{\epsilon/16}).
    \label{eq:step1-concentration}
\end{align}
By \Cref{def:link}, $\phi'(0) = \phi'(\pm 1) = 0$, and the local Taylor expansions $\phi'(t) = 2ct + O(|t|^3)$ near $0$ and $\phi'(t) = O(|1 - |t||)$ near $\pm 1$ yield a constant $C_\phi > 0$ such that
\begin{align}
    |\phi'(t)| \leq C_\phi \, |t - z^\star_i| 
    \quad \text{whenever $z^\star_i \in \{-1, 0, +1\}$ is the nearest zero of $\phi'$ to $t$.} \label{eq:step1-lipschitz}
\end{align}
Since $\bb_{c_1[\alpha]}, \bb_{c_2[\alpha]} \in \{-1, +1\}^B$, their average 
lies coordinate-wise in $\{-1, 0, +1\}$. Applying \eqref{eq:step1-lipschitz} coordinate-wise with $\eqref{eq:step1-concentration}$,
\begin{align*}
    \| \phi'(\hat \z^{(\ell'')}_{\alpha, [\ell'']}) \|_\infty 
    \leq C_\phi \, \| \hat \z^{(\ell'')}_{\alpha, [\ell'']} - \frac{\bb_{c_1[\alpha]} + \bb_{c_2[\alpha]}}{2} \|_\infty 
    \leq O(\exp(-C n^{\epsilon/16})).
\end{align*}

\paragraph{Step 2: Per-node sensitivity under the well-trained hypothesis.}
Recall the per-node sensitivity defined in \Cref{lem:per-node-sensitivity}:
\begin{align*}
    \xi^{(\ell)}_{j, m, \ell'} := \max_{\beta \,:\, h[\beta] \leq \ell + 1} 
    \left\| \frac{\partial \x^{(\ell)}_{\beta, [\ell+1]}}{\partial w^{(\ell')}_{j, m}} \right\|_\infty.
\end{align*}
We claim that for all $\ell' \leq \ell \leq t$,
\begin{align*}
    \xi^{(\ell)}_{j, m, \ell'} \leq O(\exp(-C n^{\epsilon/16})).
\end{align*}

\emph{Base case $\ell = \ell'$.} By \Cref{cor:frozen-derivative}, only $\beta$ 
with active layer equal to $\ell'$, i.e., $\nu := h[\beta] = \ell' + 1$, contribute. 
For such $\beta$,
\begin{align*}
    \frac{\partial \x^{(\ell')}_{\beta, [\ell'+1]}}{\partial w^{(\ell')}_{j, m}} 
    = \phi'(\hat \z^{(\ell')}_{\beta, [\ell']}) \odot 
    \frac{\partial \hat \z^{(\ell')}_{\beta, [\ell']}}{\partial w^{(\ell')}_{j, m}},
\end{align*}
and the parameter $w^{(\ell')}_{j, m}$ enters only through the softmax at 
query position $\beta = m$, giving
\begin{align*}
    \frac{\partial \hat \z^{(\ell')}_{m, [\ell']}}{\partial w^{(\ell')}_{j, m}} 
    = \sigma_j(\w^{(\ell')}_m) \bigl( \x^{(\ell'-1)}_{j, [\ell']} - \hat \z^{(\ell')}_{m, [\ell']} \bigr).
\end{align*}
Since $\x^{(\ell'-1)}_{j, [\ell']}, \hat \z^{(\ell')}_{m, [\ell']} \in [-1, 1]^B$, 
their difference lies in $[-2, 2]^B$. Combining with Step 1,
\begin{align*}
    \xi^{(\ell')}_{j, m, \ell'} \leq 2\, \| \phi'(\hat \z^{(\ell')}_{m, [\ell']}) \|_\infty 
    \cdot \sigma_j(\w^{(\ell')}_m) \leq O(\exp(-C n^{\epsilon/16})).
\end{align*}

\emph{Inductive step $\ell \Rightarrow \ell + 1$ ($\ell + 1 \leq t$).} 
Take $\beta$ with $h[\beta] \leq \ell + 2$.

\emph{Sub-case A: $h[\beta] \leq \ell + 1$.} By \Cref{lemma:hc-propagation}, 
$\x^{(\ell+1)}_{\beta, [\ell+2]} = \x^{(\ell)}_{\beta, [\ell+1]}$, so the 
derivative is bounded by $\xi^{(\ell)}_{j, m, \ell'}$.

\emph{Sub-case B: $h[\beta] = \ell + 2$.} \Cref{cor:frozen-derivative} gives
\begin{align*}
    \frac{\partial \x^{(\ell+1)}_{\beta, [\ell+2]}}{\partial w^{(\ell')}_{j, m}} 
    = \phi'(\hat \z^{(\ell+1)}_{\beta, [\ell+1]}) \odot 
    \frac{\partial \hat \z^{(\ell+1)}_{\beta, [\ell+1]}}{\partial w^{(\ell')}_{j, m}}.
\end{align*}
Since $\ell + 1 \leq t$, Step 1 yields $\| \phi'(\hat \z^{(\ell+1)}_{\beta, [\ell+1]}) \|_\infty 
\leq O(\exp(-C n^{\epsilon/16}))$. Differentiation passes through the softmax 
(since $\ell' \leq \ell < \ell + 1$):
\begin{align*}
    \frac{\partial \hat \z^{(\ell+1)}_{\beta, [\ell+1]}}{\partial w^{(\ell')}_{j, m}} 
    = \sum_{\gamma} \sigma_\gamma(\w^{(\ell+1)}_\beta) \, 
    \frac{\partial \x^{(\ell)}_{\gamma, [\ell+1]}}{\partial w^{(\ell')}_{j, m}},
\end{align*}
each summand bounded by $\xi^{(\ell)}_{j, m, \ell'}$ in $\infty$-norm and 
the softmax weights summing to $1$. Hence
\begin{align*}
    \left\| \frac{\partial \x^{(\ell+1)}_{\beta, [\ell+2]}}{\partial w^{(\ell')}_{j, m}} \right\|_\infty 
    \leq O(\exp(-C n^{\epsilon/16})) \cdot \xi^{(\ell)}_{j, m, \ell'}.
\end{align*}

\emph{Combining sub-cases.} 
\begin{align*}
    \xi^{(\ell+1)}_{j, m, \ell'} \leq \xi^{(\ell)}_{j, m, \ell'} \leq O(\exp(-C n^{\epsilon/16})),
\end{align*}
where the second inequality is the inductive hypothesis.

\paragraph{Step 3: Bounding the layer-$(t + 1)$ output gradient.}
For any $\alpha \in [n_{t+1} + 1, T]$, the softmax weights at layer $t + 1$ 
are independent of $w^{(\ell')}_{j, m}$ as layer $t + 1$ has not yet been 
trained. By chain rule,
\begin{align*}
    \frac{\partial \phi(\hat \z^{(t+1)}_{\alpha, [t+1]})}{\partial w^{(\ell')}_{j, m}} 
    = \phi'(\hat \z^{(t+1)}_{\alpha, [t+1]}) \odot \sum_\gamma 
    \sigma_\gamma(\w^{(t+1)}_\alpha) \cdot 
    \frac{\partial \x^{(t)}_{\gamma, [t+1]}}{\partial w^{(\ell')}_{j, m}}.
\end{align*}
Bounding by $\|\phi'\|_\infty \leq L_\phi$, summing softmax weights to $1$, 
and using Step 2's bound $\xi^{(t)}_{j, m, \ell'} \leq O(\exp(-C n^{\epsilon/16}))$,
\begin{align*}
    \left\| \frac{\partial \phi(\hat \z^{(t+1)}_{\alpha, [t+1]})}{\partial w^{(\ell')}_{j, m}} \right\|_\infty 
    \leq L_\phi \cdot \xi^{(t)}_{j, m, \ell'} \leq O(\exp(-C n^{\epsilon/16})).
\end{align*}
\paragraph{Step 4: Pulling back through the loss.}
The stage-$(t+1)$ loss reads
\begin{align*}
    \Lc^{(t+1)}(\bte) = \frac{1}{2B} \sum_{\alpha = n_{t+1} + 1}^{T} 
    \| \Pb_{t+1} \left(g_\alpha^{(t + 1)} \phi(\hat \z^{(t+1)}_\alpha) 
    + (1 - g_\alpha^{(t + 1)}) \x_\alpha^{(t)} \right)  - \bb_{\alpha} \|^2.
\end{align*}
Splitting the sum by the gating rule Eq.~\eqref{eq:gating-weights},  
$g_\alpha^{(t+1)} = 1$ for $\alpha \in [n_{t+1}+1, n_{t+2}]$ and 
$g_\alpha^{(t+1)} = 0$ for $\alpha > n_{t+2}$ and differentiating,
\begin{align*}
    &\left| \frac{\partial \Lc^{(t+1)}}{\partial w^{(\ell')}_{j, m}} \right| \\
    &= \frac{1}{B} \left| \sum_{\alpha = n_{t+1}+1}^{n_{t+2}} 
    \bigl( \Pb_{t+1} \phi(\hat \z^{(t+1)}_\alpha) - \bb_{\alpha} \bigr)^\top 
    \frac{\partial \Pb_{t+1} \phi(\hat \z^{(t+1)}_\alpha)}{\partial w^{(\ell')}_{j, m}} 
    + \sum_{\alpha = n_{t+2}+1}^{T} 
    \bigl( \Pb_{t+1} \x_\alpha^{(t)} - \bb_{\alpha} \bigr)^\top 
    \frac{\partial \Pb_{t+1} \x_\alpha^{(t)}}{\partial w^{(\ell')}_{j, m}} \right| \\
    &\overset{(a)}{\leq} \sum_{\alpha = n_{t+1}+1}^{n_{t+2}} \frac{1}{B} 
    \| \Pb_{t+1} \phi(\hat \z^{(t+1)}_\alpha) - \bb_{\alpha} \|_1 \cdot 
    \left\| \frac{\partial \phi(\hat \z^{(t+1)}_{\alpha, [t+1]})}{\partial w^{(\ell')}_{j, m}} \right\|_\infty \\
    &\qquad + \sum_{\alpha = n_{t+2}+1}^{T} \frac{1}{B} 
    \| \Pb_{t+1} \x_\alpha^{(t)} - \bb_{\alpha} \|_1 \cdot 
    \left\| \frac{\partial \x^{(t)}_{\alpha, [t+1]}}{\partial w^{(\ell')}_{j, m}} \right\|_\infty \\
    &\overset{(b)}{\leq} 2 \sum_{\alpha = n_{t+1}+1}^{n_{t+2}} 
    \left\| \frac{\partial \phi(\hat \z^{(t+1)}_{\alpha, [t+1]})}{\partial w^{(\ell')}_{j, m}} \right\|_\infty 
    + 2 \sum_{\alpha = n_{t+2}+1}^{T} 
    \left\| \frac{\partial \x^{(t)}_{\alpha, [t+1]}}{\partial w^{(\ell')}_{j, m}} \right\|_\infty \\
    &\overset{(c)}{\leq} O(\exp(-C n^{\epsilon/16})),
\end{align*}
where (a) applies H\"older's inequality and uses that 
$\Pb_{t+1}$ extracts block $[t+1]$, (b) uses 
$\Pb_{t+1} \phi(\hat \z^{(t+1)}_\alpha)$, $ \Pb_{t+1} \x_\alpha^{(t)}, 
\bb_{\alpha} \in [-1, 1]^B$, and (c) bounds the first sum by Step 3 and the 
second sum by the fact that positions $\alpha \geq n_{t + 2} + 1$ has not been written by layer $\ell' \in [t]$. 
The factor of $T \leq 2n$ from each sum is absorbed into the exponential.

\paragraph{Step 5: Quantization freezes the weight.} {With the stage-wise learning rate $\eta_{t+1} = \Theta(K_n n_{t+1}^2)$ 
where $K_n = \lceil n^{\epsilon/16} \rceil$,} the proposed update reads
\begin{align*}
    w^{(\ell')}_{j, m}(t) - \eta_{t+1} \, \widetilde\nabla_{w^{(\ell')}_{j, m}} \Lc^{(t+1)} 
    & = w^{(\ell')}_{j, m}(t) + {O(K_n n_{t+1}^2 \left ( \exp(-C n^{\epsilon/16}) + n^{-2 - \epsilon / 8}\right)}
    \\ & = w^{(\ell')}_{j, m}(t) + O(n^{-\epsilon / 16}),
\end{align*}
which the nearest-integer operator $q(\cdot)$ rounds back to $w^{(\ell')}_{j, m}(t)$ for sufficiently large $n$.
\end{proof}

\end{document}